\documentclass[letterpaper]{article} 
\usepackage{aaai2027}
\usepackage[hyphens]{url}  
\usepackage{graphicx} 
\usepackage{natbib}  
\usepackage{caption} 
\usepackage{algorithm}
\usepackage{algorithmic}

\usepackage{amsmath}
\usepackage{amssymb}
\usepackage{mathtools}
\usepackage{amsthm}

\usepackage{subcaption}

\theoremstyle{plain}
\newtheorem{theorem}{Theorem}
\newtheorem{proposition}[theorem]{Proposition}
\newtheorem{lemma}[theorem]{Lemma}

\theoremstyle{definition}

\theoremstyle{remark}

\newcommand{\lammax}{\lambda_{\max}}

\usepackage{newfloat}
\usepackage{listings}
\DeclareCaptionStyle{ruled}{labelfont=normalfont,labelsep=colon,strut=off} 
\floatstyle{ruled}
\newfloat{listing}{tb}{lst}{}
\floatname{listing}{Listing}

\usepackage{booktabs}

\nocopyright

\title{Weight-norm Criticality: A Mechanism for Loss Spikes Induced by the Normalization and Weight Decay}
\author{
    Xiaolong Li\textsuperscript{\rm 1},
    Zhangchen Zhou\textsuperscript{\rm 1}\equalcontrib\corresponding,
    Zhi-Qin John Xu\textsuperscript{\rm 1}\equalcontrib\corresponding
}
\affiliations{
    \textsuperscript{\rm 1}Institute of Natural Sciences, School of Mathematical Sciences, Shanghai Jiao Tong University
}

\begin{document}

\maketitle

\begingroup
\renewcommand{\thefootnote}{}
\footnotetext{%
    \normalfont\footnotesize
    \textit{E-mail:}
    \{15369855310, zczhou1115, xuzhiqin\}@sjtu.edu.cn
}
\endgroup

\begin{abstract}
Most explanations of training instability focus on \emph{learning-rate criticality}, typically characterized by the Edge of Stability, beyond which optimization becomes unstable. We argue that, in practical deep neural network training, there is an additional and often overlooked \emph{weight-norm criticality}. This criticality is induced by the interaction between normalization (which introduces scale-invariant components) and weight decay (which persistently shrinks parameter norms). As the weight decay coefficient increases, the norms of scale-invariant weights are progressively driven toward zero. Meanwhile, the sharpness of the loss landscape increases rapidly, destabilizing the optimization dynamics and resulting in abrupt loss spikes. This perspective provides a rationale for why weight penalties can improve generalization yet cannot be made arbitrarily strong: excessive decay drives scale-invariant weight norms past a critical boundary and destabilizes training. Our work provides a new mechanistic understanding of loss spikes through the lens of \emph{weight-norm criticality}. Moreover, \emph{weight-norm criticality} yields testable predictions that we validate empirically in networks with scale-invariant components, providing empirical support for the proposed mechanism.
\end{abstract}


\section{Introduction}

\

Loss spikes, characterized by sudden increases in training loss, are commonly observed in deep neural network training. Their causes can be multifaceted, including: (i) \emph{data}, where heterogeneous mini-batches produce noisy gradient directions~\citep{chowdhery2023palm}; (ii) \emph{landscape}, where lower-loss regions become sharper and are more likely to cross the stability threshold~\citep{li2023loss}; and (iii) \emph{optimization}, where Adam may amplify the adaptive learning rate when the second-moment estimate $v_t$ decreases more slowly than the gradient $g_t$~\citep{bai2025adaptive}.


\begin{figure}[t]
  \centering
  \begin{subfigure}[t]{0.22\textwidth}
    \centering
    \includegraphics[width=\linewidth]{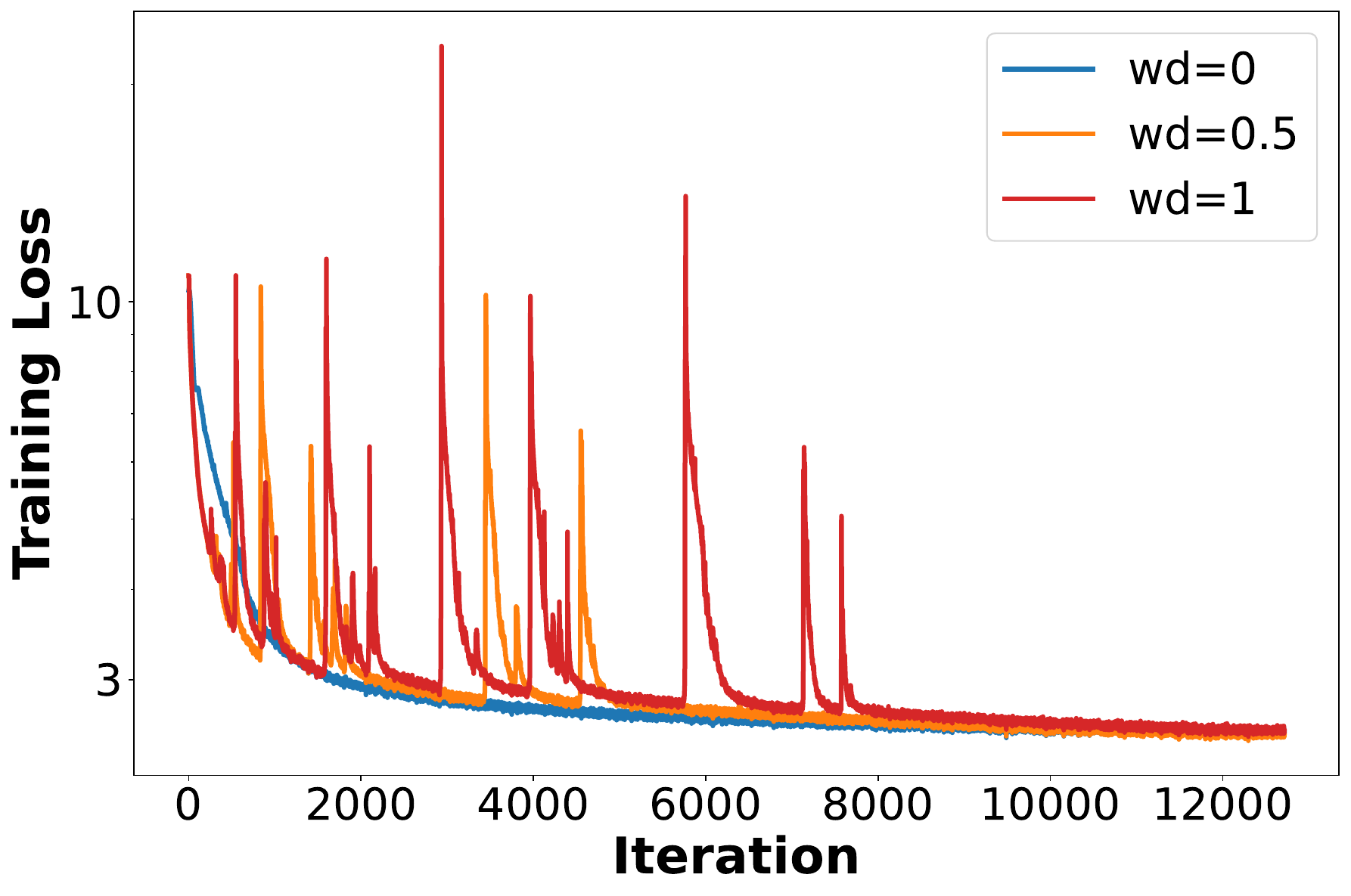}
    \caption{Transformer}
    \label{fig:transformer_100b}
  \end{subfigure}
  \hfill
  \begin{subfigure}[t]{0.22\textwidth}
    \centering
    \includegraphics[width=\linewidth]{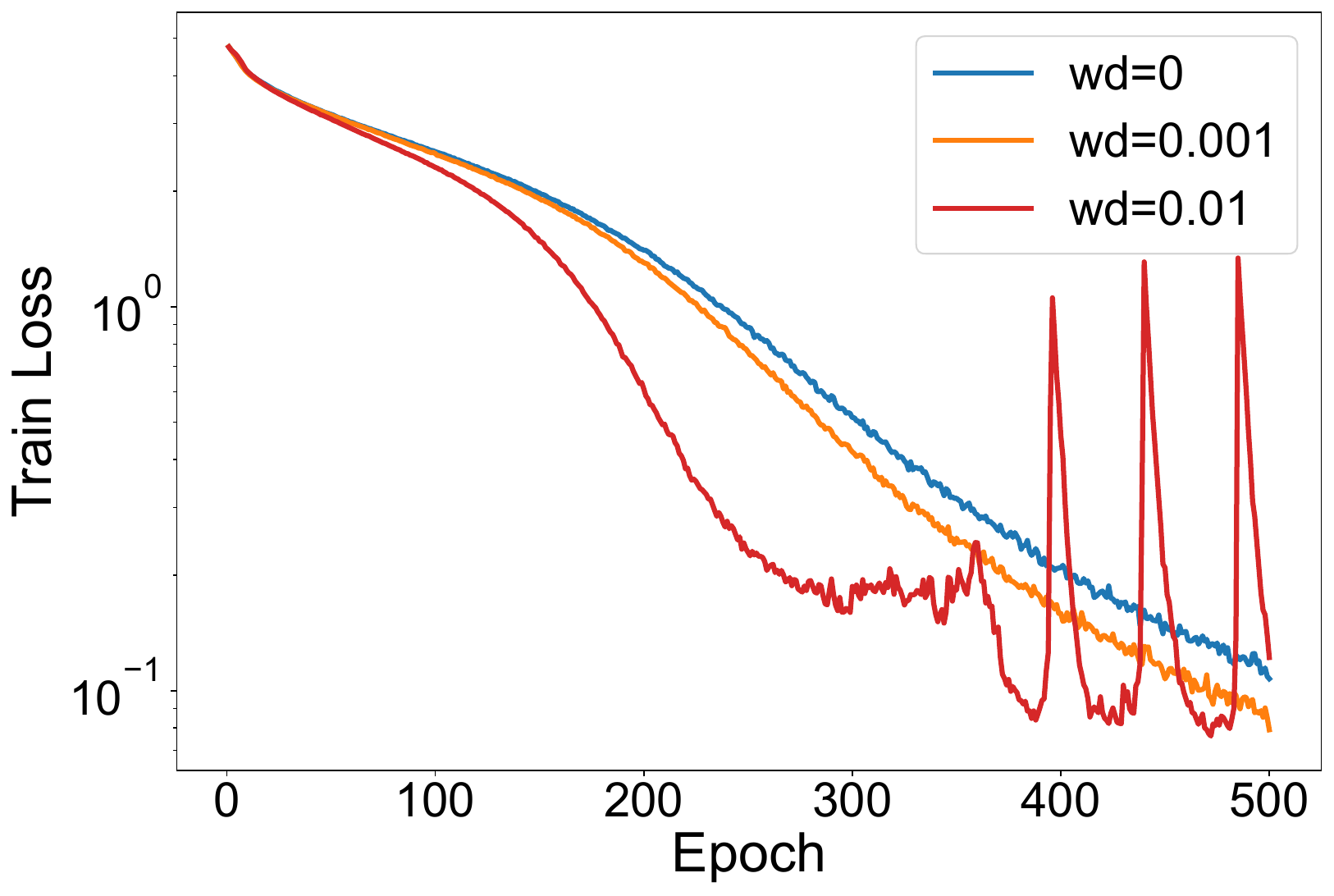}
    \caption{ResNet-50}
    \label{fig:ResNet}
  \end{subfigure}
\caption{\textbf{The impact of weight decay $\lambda$ on training loss.} Larger weight decay leads to more frequent loss spikes. \textbf{(a)} Training loss trajectories of a 16-layer, 16-head 187M Transformer on a 100B-token corpus. \textbf{(b)} ResNet-50 trained on the CIFAR-100 dataset.}
  \label{fig:spike-real-data}
\end{figure}

Much of the existing literature explains loss spikes through the Edge of Stability (EoS)~\citep{cohen2021gradient}, where \emph{learning-rate criticality} is viewed as the main cause: spikes occur when the effective step size crosses a stability boundary. In this work, we argue that practical training exhibits a frequently overlooked \emph{weight-norm criticality} within specific architectures. This phenomenon provides a complementary perspective to \emph{learning-rate criticality}, collectively explaining the emergence of loss spikes.

Concretely, we study the interaction between normalization (e.g., BatchNorm (BN)~\citep{ioffe2015batch} and LayerNorm (LN)~\citep{ba2016layer}) and weight decay, and show that it can readily induce loss spikes, as demonstrated by both a Transformer network~\citep{vaswani2017attention} and a ResNet-50~\citep{he2016deep} in Fig.~\ref{fig:spike-real-data}. The key mechanism is scale invariance: normalization makes the loss insensitive to rescaling certain weights, while the local curvature remains highly sensitive to their norms.

We first examine where training converges in the loss landscape. As weight decay becomes stronger, normalized networks converge to increasingly sharp regions. In particular, the norms of scale-invariant parameters induced by normalization tend to shrink toward zero. We theoretically show that this norm shrinkage amplifies the corresponding Hessian eigenvalues, causing a rapid increase in local sharpness. Based on this scaling effect, we derive a stability indicator for \emph{weight-norm criticality} that combines the Hessian curvature and learning rate to predict the occurrence of loss spikes.




\begin{figure*}[t] 
  \centering
  \begin{subfigure}[t]{0.33\textwidth}
    \centering
    \includegraphics[width=0.8\linewidth]{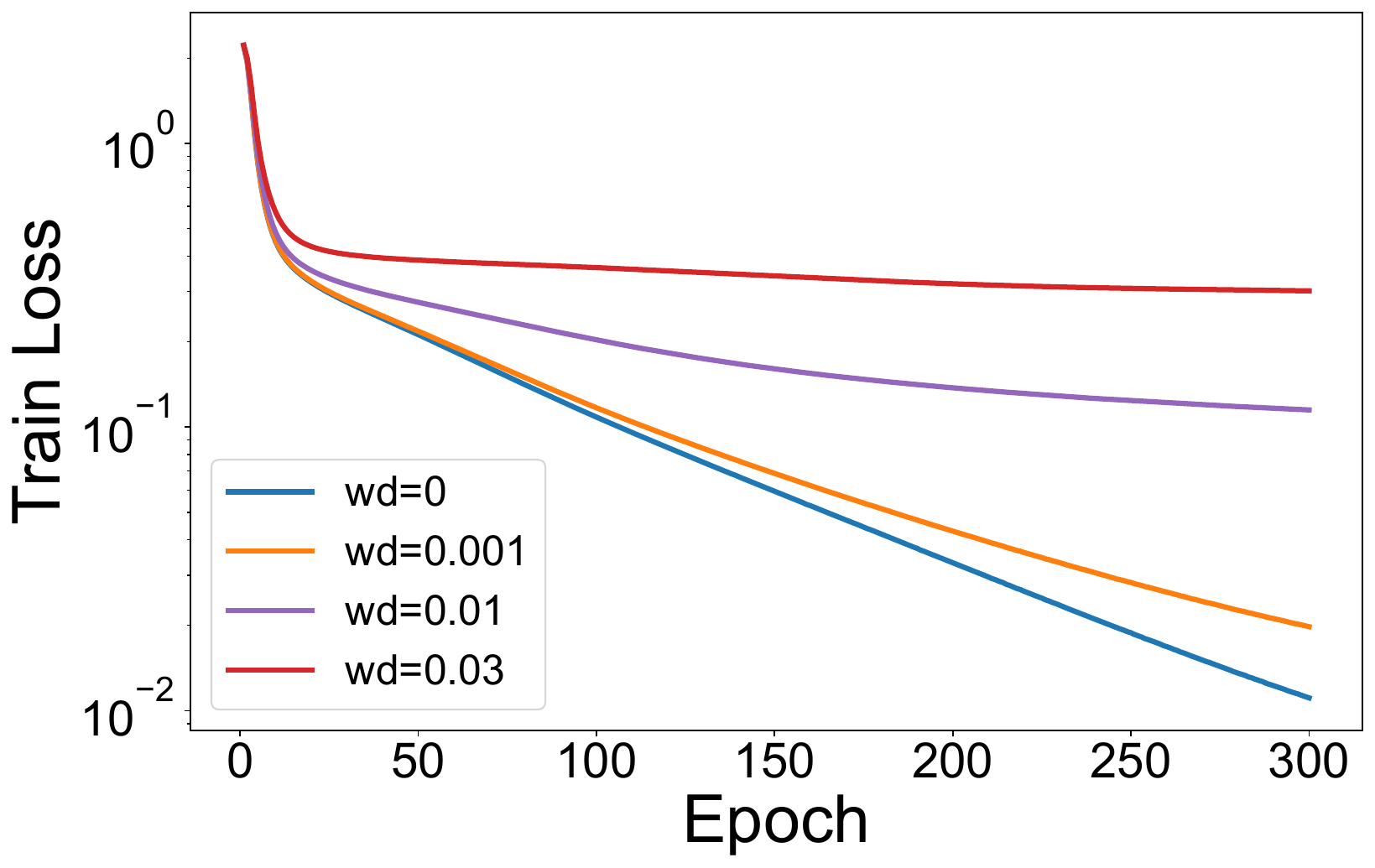}
    \caption{\textbf{Identity}}
    \label{fig:none_003}
  \end{subfigure}
  \hfill
  \begin{subfigure}[t]{0.32\textwidth}
    \centering
    \includegraphics[width=0.8\linewidth]{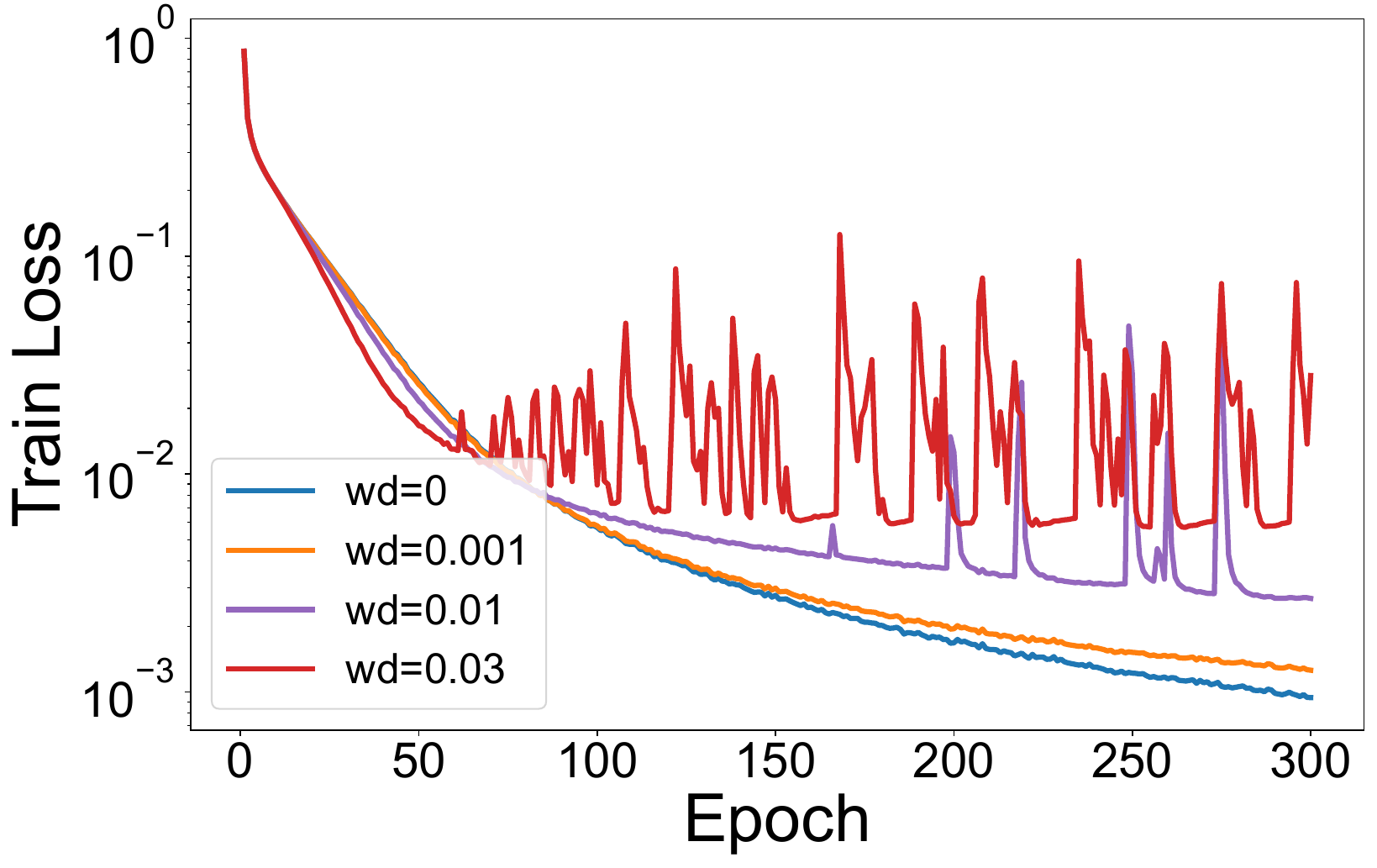}
    \caption{\textbf{BatchNorm (BN)}}
    \label{fig:bn_003}
  \end{subfigure}
  \hfill
  \begin{subfigure}[t]{0.33\textwidth}
    \centering
    \includegraphics[width=0.8\linewidth]{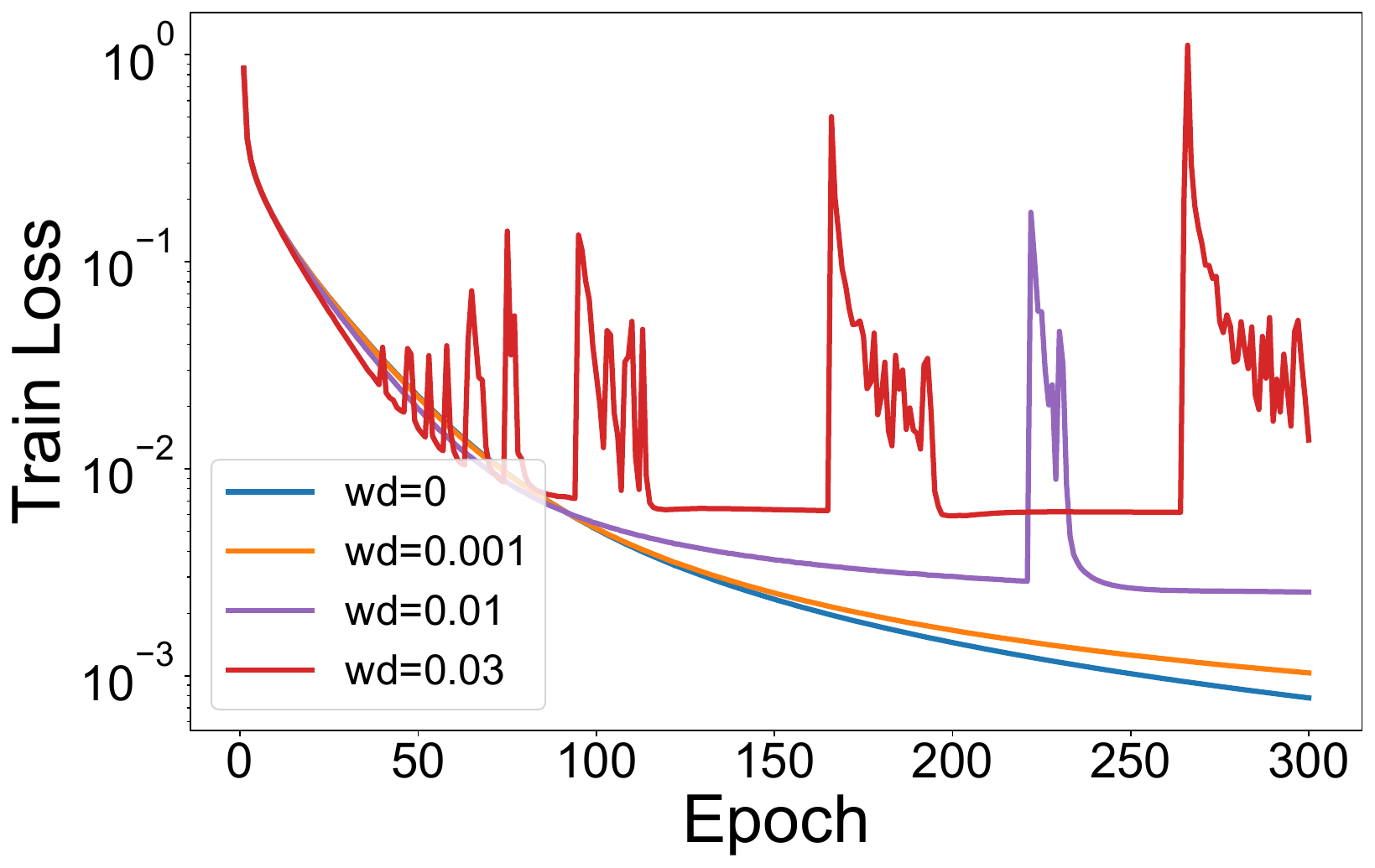}
    \caption{\textbf{LayerNorm (LN)}}
    \label{fig:ln_003}
  \end{subfigure}

  \caption{\textbf{Mechanistic probe on MNIST with optional normalization.}
Training loss trajectories of fully connected neural networks trained on MNIST with SGD under a sweep of the weight decay coefficient.
\textbf{(a) Identity.} \texttt{Linear}$\rightarrow$\texttt{tanh} hidden blocks.
\textbf{(b) BatchNorm (BN).} \texttt{Linear}$\rightarrow$\texttt{BN}$\rightarrow$\texttt{tanh} hidden blocks.
\textbf{(c) LayerNorm (LN).} \texttt{Linear}$\rightarrow$\texttt{LN}$\rightarrow$\texttt{tanh} hidden blocks.}
  \label{fig:combined_four_plots}
\end{figure*}

Moreover, the critical boundary can be decomposed across individual scale-invariant components. For a given component, computing its \emph{weight-norm criticality} only requires the Hessian restricted to that component, rather than the full-network Hessian. This yields a separate critical threshold for each scale-invariant layer and naturally extends our analysis to practical architectures with multiple normalization or other scale-invariant modules. These layer-wise thresholds also enable instability localization: when a loss spike is caused by weight-norm collapse, the responsible layers or components can be identified. This also highlights a practical trade-off: although weight decay can improve generalization, excessive decay may drive scale-invariant weight norms across the critical boundary and destabilize training.

This work makes three main contributions. First, we show that weight decay can induce loss spikes even in networks that are not globally scale-invariant, as long as they contain scale-invariant components. Second, we characterize this instability through \emph{weight-norm criticality}, linking loss spikes to the shrinking norms of these components. Third, this framework provides layer- and component-wise diagnostics, allowing the source of instability to be localized within practical models.

\section{Related works}

\paragraph{Loss spikes analysis}
Several studies have investigated loss spikes from the perspectives of the loss landscape, optimization, and data. \citet{li2023loss} attributed certain spikes to sharp regions exhibiting a ``lower-loss-as-sharper'' structure. From the optimization perspective, \citet{ma2022qualitative} linked Adam's hyperparameters to spikes and oscillations, while \citet{molybog2023theory} observed that gradients and second-moment estimates in shallow layers can decay toward zero and spike abruptly upon encountering large gradients. More recently, \citet{bai2025adaptive} showed that Adam's adaptivity can itself trigger spikes and that reducing $\beta_2$ can mitigate them. From the data perspective, \citet{chowdhery2023palm} reported that restarting from an earlier checkpoint while skipping a problematic batch can resolve training instabilities in large-scale models.

\paragraph{Edge of Stability}
Extensive studies~\citep{wu2018sgd, xing2018walk, Jastrzebski2018on, Jastrzebski2020The, lewkowycz2020large, cohen2021gradient, ahn2022understanding, NEURIPS2022_dffd1c52, arora2022understanding} have shown that neural network training often operates near a sharpness threshold, a phenomenon termed the ``Edge of Stability'' (EoS) by \citet{cohen2021gradient}. \citet{NEURIPS2022_dffd1c52} extended this perspective to scale-invariant models, showing that the effective learning rate exhibits EoS behavior with respect to spherical sharpness, but not common sharpness. In this regime, \citet{damian2023selfstabilization} and \citet{wang2022analyzing} showed that when $\lambda_{\max}>2/\eta$, self-stabilization can suppress sharpness and restore stability.

\paragraph{Interplay of weight decay and normalization}
Prior work has shown that, in scale-invariant settings, the interaction between normalization and weight decay affects optimization through the weight scale and effective learning rate~\citep{van2017l2, NEURIPS2018_a0160709, zhang2018three, NEURIPS2020_a7453a5f}. This interaction can drive training toward a stable equilibrium~\citep{NEURIPS2019_cb3ce9b0, NEURIPS2020_a7453a5f, NEURIPS2021_326a8c05}, but can also cause instability. \citet{li2019exponential} showed that globally scale-invariant functions become ill-conditioned as weights approach the origin, and \citet{NEURIPS2020_a7453a5f} empirically verified related instability under full-batch GD. \citet{Li_Chen_Yang_2020} further showed that excessive weight decay can rapidly shrink layer weights, causing the effective learning rate to diverge and training to fail. Related studies have also reported periodic dynamics~\citep{NEURIPS2021_b433da1b} and distinct convergence, chaotic, and divergence regimes on the sphere~\citep{NEURIPS2022_5aea56ee}. While these analyses mainly focus on globally scale-invariant models, we consider architectures containing both scale-invariant and non-scale-invariant parameters, which better reflect practical deep networks.

\section{Learning Trajectory: Convergence toward Singularities on the Loss Landscape}
\label{sec:Learning Trajectory}

\

We first establish empirically that increasing weight decay can induce pronounced loss spikes in neural networks equipped with normalization layers. This phenomenon is consistently observed across different architectures and tasks. Specifically, we pretrain a LLaMA-style~\citep{touvron2023llama} Transformer on a 100B-token corpus while varying only the weight decay coefficient $\lambda \in {0,0.5,1}$ (Fig.~\ref{fig:transformer_100b}), train a ResNet-50 on CIFAR-100 with SGD under different weight decay strengths (Fig.~\ref{fig:ResNet}), and conduct a controlled study on MNIST using fully connected networks with and without BN or LN (Fig.~\ref{fig:combined_four_plots}). Detailed experimental settings are provided in Appendix~B. Across the Transformer, ResNet, and fully connected architectures equipped with normalization, stronger weight decay consistently leads to increasingly pronounced loss spikes (Figs.~\ref{fig:transformer_100b}, \ref{fig:ResNet}, \ref{fig:bn_003} and \ref{fig:ln_003}). In contrast, in the controlled MNIST experiments, the corresponding networks without normalization do not exhibit the same spike behavior as weight decay is increased (Fig.~\ref{fig:none_003}). These observations suggest that the loss spikes cannot be attributed to weight decay or normalization in isolation; rather, they arise from the combined effect of weight decay and normalization.

Motivated by these observations, we next examine how weight decay alters the optimization trajectory in parameter space. To compare the optimization dynamics under different weight decay coefficients $\lambda$, we construct a shared two-dimensional parameter subspace using principal component analysis (PCA). At each epoch, all trainable parameters are flattened into a vector $\boldsymbol{\theta}_i$, and the trajectories from all epochs and all $\lambda$ settings are jointly used to compute the top two principal directions. The resulting two-dimensional PCA subspace is therefore common to all runs, enabling direct comparison of their optimization trajectories.

To visualize the corresponding loss landscape, we evaluate the training loss on a two-dimensional grid in this PCA subspace. Each grid point is mapped back to the original parameter space through the inverse PCA transformation, and the reconstructed parameters are used to evaluate the loss. The resulting values define the loss contours shown in the figures. All runs share the same initialization, ensuring that their projected trajectories start from the same point. Fig.~\ref{fig:mnist_sifc_landscape_traj} shows the resulting PCA loss landscape together with the projected training trajectories.


\begin{figure}[h]
  \centering
\includegraphics[width=0.9\columnwidth]{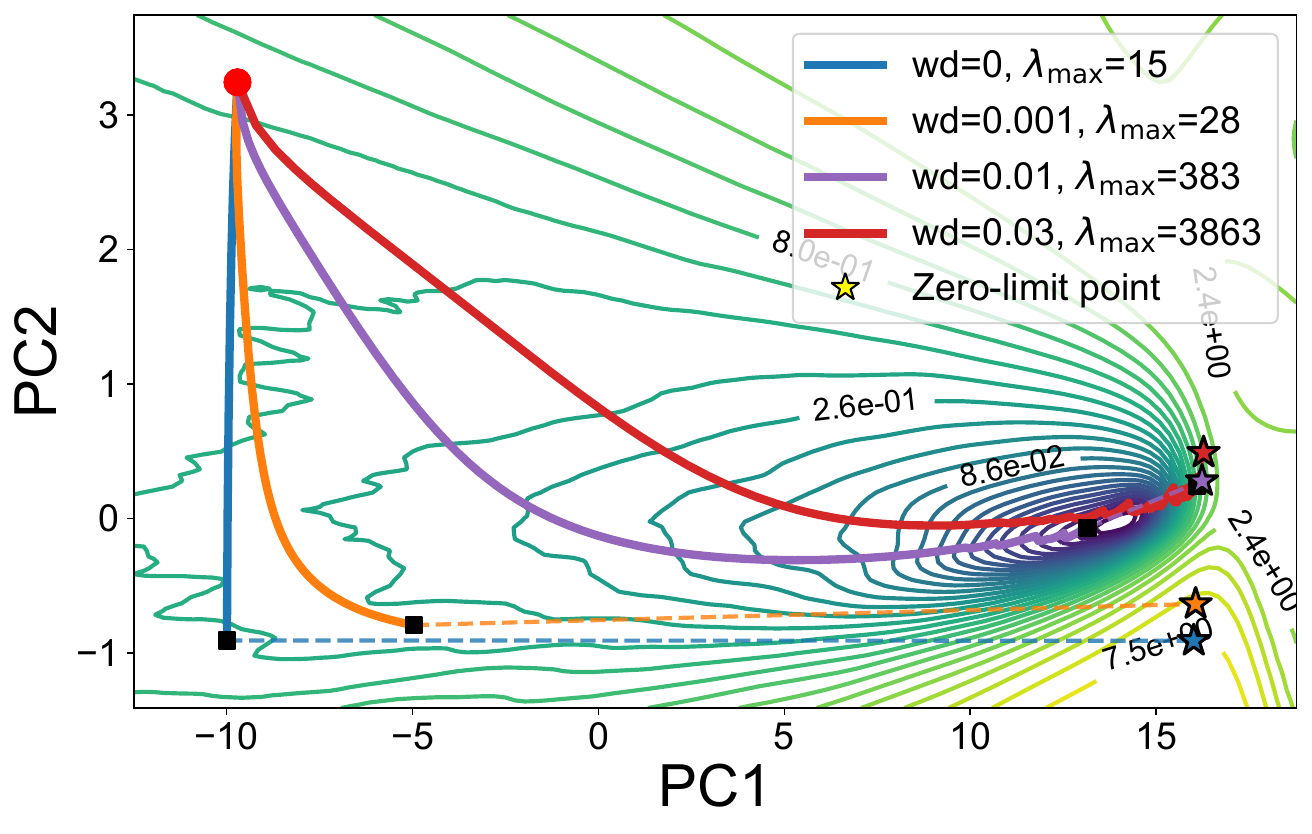}
    \caption{\textbf{PCA visualization of training trajectories} for a five-layer FNN with BN on MNIST under varying weight decay. The learning rate is fixed to be $\eta=0.003$. The black squares mark the final parameter states and $\lambda_{\max}$ denotes the largest eigenvalue of the Hessian matrix at the training endpoint. The colored stars represent the derived parameters obtained by freezing the output layer of the trained models and pushing the scale-invariant parameters to zero. Contour lines represent training loss. The explained variance ratios of PC1 and PC2 are $0.9927$ and $0.0059$, respectively.}
  \label{fig:mnist_sifc_landscape_traj}
\end{figure}

Fig.~\ref{fig:mnist_sifc_landscape_traj} highlights two distinct phenomena. First, increasing weight decay leads to a larger maximum eigenvalue at convergence, suggesting that stronger regularization biases the optimization toward sharper solutions. Second, the converged solutions under different $\lambda$ values do not scatter arbitrarily but instead exhibit a systematic displacement along a particular direction in parameter space. To characterize this, we define a zero-limit reference point ($\star$) for each run by setting the parameters of scale-invariant layers to zero while keeping others fixed. Notably, the solution shift aligns well with the direction toward this zero-limit reference point. This indicates that weight decay does not uniformly contract the parameter vector; instead, it predominantly suppresses the scale-invariant layers, resulting in a structured displacement in parameter space.

An additional observation is that these zero-limit reference points consistently lie in regions of extremely large curvature on the loss landscape. This motivates us to examine whether the two phenomena described above are intrinsically connected. In particular, as the weight decay coefficient increases, the converged solution is progressively displaced toward the low-dimensional manifold associated with vanishing scale-invariant parameters, where the local curvature is significantly larger. As the distance to this zero-limit manifold decreases, the solution becomes increasingly sharp, as reflected by the rapid growth of $\lambda_{\max}$. We conjecture that this directional drift toward high-curvature regions can eventually destabilize the optimization dynamics, giving rise to the observed loss spikes. In the following sections, we provide a theoretical analysis to substantiate this interpretation.

\section{Theoretical Analysis: Scale Invariance and Curvature Dynamics}


\

In this section, we analyze the homogeneity properties of the Hessian induced by scale invariance and study how parameter scaling affects curvature. Based on this analysis, we derive a stability boundary expressed directly in terms of the weight norm of scale-invariant components, which reveals a weight-norm criticality in this setting and provides a mechanistic explanation for the emergence of loss spikes observed in the experiments. Finally, we validate this theoretical explanation through controlled experiments on practical learning tasks with real data, demonstrating that violations of the derived stability boundary coincide with the onset of loss spikes in practice.


\subsection{Homogeneity of the Hessian Matrix}

\

Consider a loss function $L(u, v)$ where parameters are partitioned into a scale-invariant component $u \in \mathbb{R}^{d_u}$ (e.g., weights preceding normalization layers) and other parameters $v \in \mathbb{R}^{d_v}$. The positive scale invariance property implies that for any scalar $\alpha > 0$:
\begin{equation}
    L(\alpha u, v) = L(u, v).
\end{equation}

\begin{theorem}[Curvature Explosion Induced by Scale Invariance]
\label{thm:curvature_explosion}


Let \(H(u,v)=\nabla^2 L(u,v)\) denote the Hessian matrix with respect to the
concatenated parameter vector \((u,v)\). We partition the Hessian into blocks
corresponding to the two parameter groups:
\[
H(u,v)=
\begin{bmatrix}
H_{uu}(u,v) & H_{uv}(u,v)\\
H_{vu}(u,v) & H_{vv}(u,v)
\end{bmatrix},
\]
where
\[
\begin{aligned}
H_{uu}(u,v)
&:= \nabla_u^2 L(u,v)
\in \mathbb{R}^{d_u\times d_u},\\
H_{vv}(u,v)
&:= \nabla_v^2 L(u,v)
\in \mathbb{R}^{d_v\times d_v},\\
H_{uv}(u,v)
&= H_{vu}^{\top}(u,v)
:= \nabla_u \nabla_v^{\top} L(u,v)
\in \mathbb{R}^{d_u\times d_v}.
\end{aligned}
\]
Here, \(H_{uu}\) and \(H_{vv}\) are the Hessian blocks with respect to the
scale-invariant parameters \(u\) and the remaining parameters \(v\),
respectively, while \(H_{uv}\) is the cross-Hessian block. For any \(\alpha>0\),
\begin{equation}
\label{eq:curvature_lb_split}
\lambda_{\max}\!\big(H(\alpha u,v)\big)
\ge
\alpha^{-2}\lambda_{\max}\!\big(H_{uu}(u,v)\big).
\end{equation}
\end{theorem}

\begin{proof}
Fix any $\alpha>0$.
By Appendix Lemma~1, the $uu$-block satisfies
{\small\begin{equation}
\label{eq:Huu_scaling_thm}
H_{uu}(\alpha u,v)=\alpha^{-2} H_{uu}(u,v).
\end{equation}}

Since $H(\alpha u,v)$ is real symmetric, $H_{uu}(\alpha u,v)$ is a real symmetric
principal submatrix of $H(\alpha u,v)$.
By the Cauchy interlacing theorem~\citep{fisk2005very},
the largest eigenvalue of the full Hessian is lower bounded by that of any
principal submatrix, yielding
{\small\begin{equation}
\label{eq:interlacing_thm}
\lammax\!\big(H(\alpha u,v)\big)
\ \ge\
\lammax\!\big(H_{uu}(\alpha u,v)\big).
\end{equation}}

Combining \eqref{eq:Huu_scaling_thm} with the homogeneity of eigenvalues for symmetric matrices,
$\lammax(\alpha M)=\alpha\,\lammax(M)$ for $\alpha>0$, we obtain
\[
\lammax\!\big(H_{uu}(\alpha u,v)\big)
=
\alpha^{-2}\,\lammax\!\big(H_{uu}(u,v)\big).
\]
Substituting this into \eqref{eq:interlacing_thm} proves
\[
\lammax\!\big(H(\alpha u,v)\big)
\ \ge\
\alpha^{-2}\,\lammax\!\big(H_{uu}(u,v)\big),
\]
which establishes the claimed lower bound.
\end{proof}

\begin{figure}[t]
    \centering
    \includegraphics[
        height=0.20\textheight
    ]{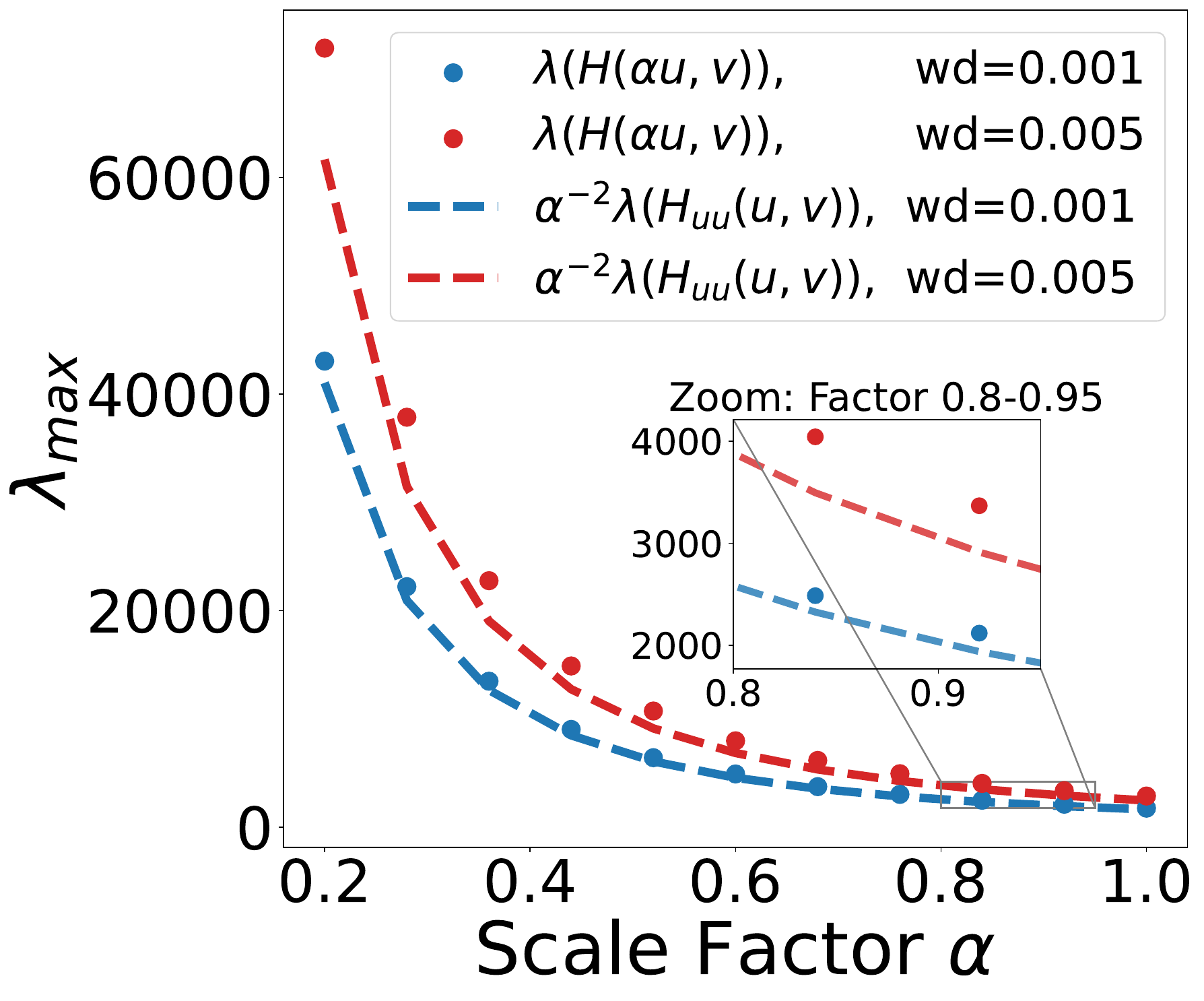}
    \caption{\textbf{Verification of the theoretical lower bound via parameter scaling in ResNet-50.}
    ResNet-50 is trained on the CIFAR-100 dataset with a fixed learning rate of $\eta = 0.002$.
    Scatters denote the top eigenvalues of the full Hessian evaluated after scaling the scale-invariant parameter components. 
    The curve represents the corresponding scaling behavior of the top eigenvalue of the Hessian restricted to the scale-invariant components, empirically validating the theoretical lower bound predicted by the scaling analysis.}
    \label{fig:resnet_scale}
\end{figure}


\begin{figure}[t]
\centering
\includegraphics[width=0.8\columnwidth]{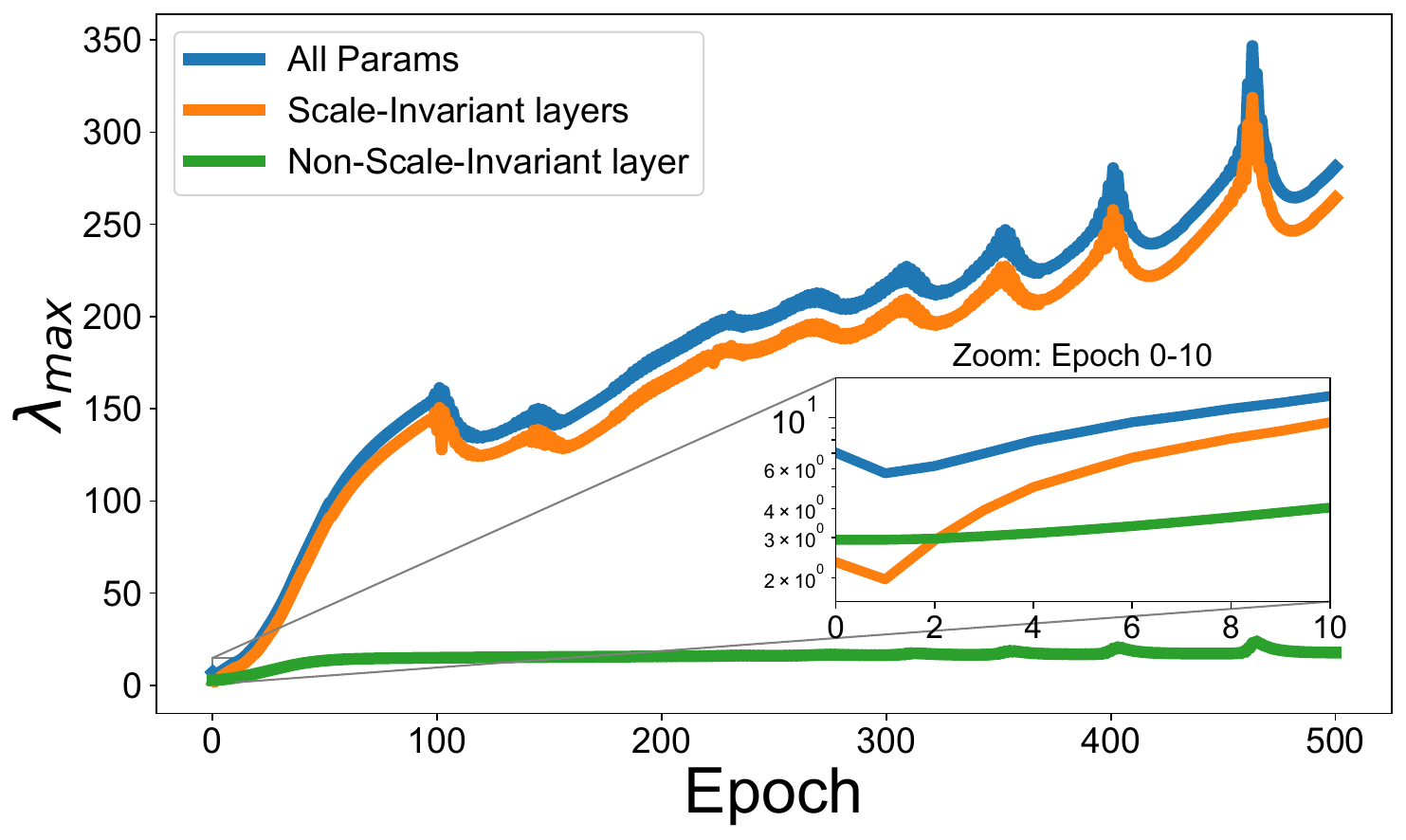}
    \caption{\textbf{Evolution of sharpness.} Three-layer FNN with BN on synthetic data, trained with $\eta=0.03$ and $\lambda=0.1$. During training, $\lambda_{\max}$ of the scale-invariant Hessian block (orange) increases rapidly and approaches that of the full Hessian (blue), while remaining much larger than that of the non-scale-invariant block (green). The inset shows that the two blocks have comparable sharpness at the beginning of training.}
    \label{hessian_trajectory_with_zoom}
\end{figure}

\begin{figure*}[t]
  \centering
  \begin{subfigure}{0.47\textwidth}
    \centering
    \includegraphics[width=0.9\linewidth,
    height=0.22\textheight]{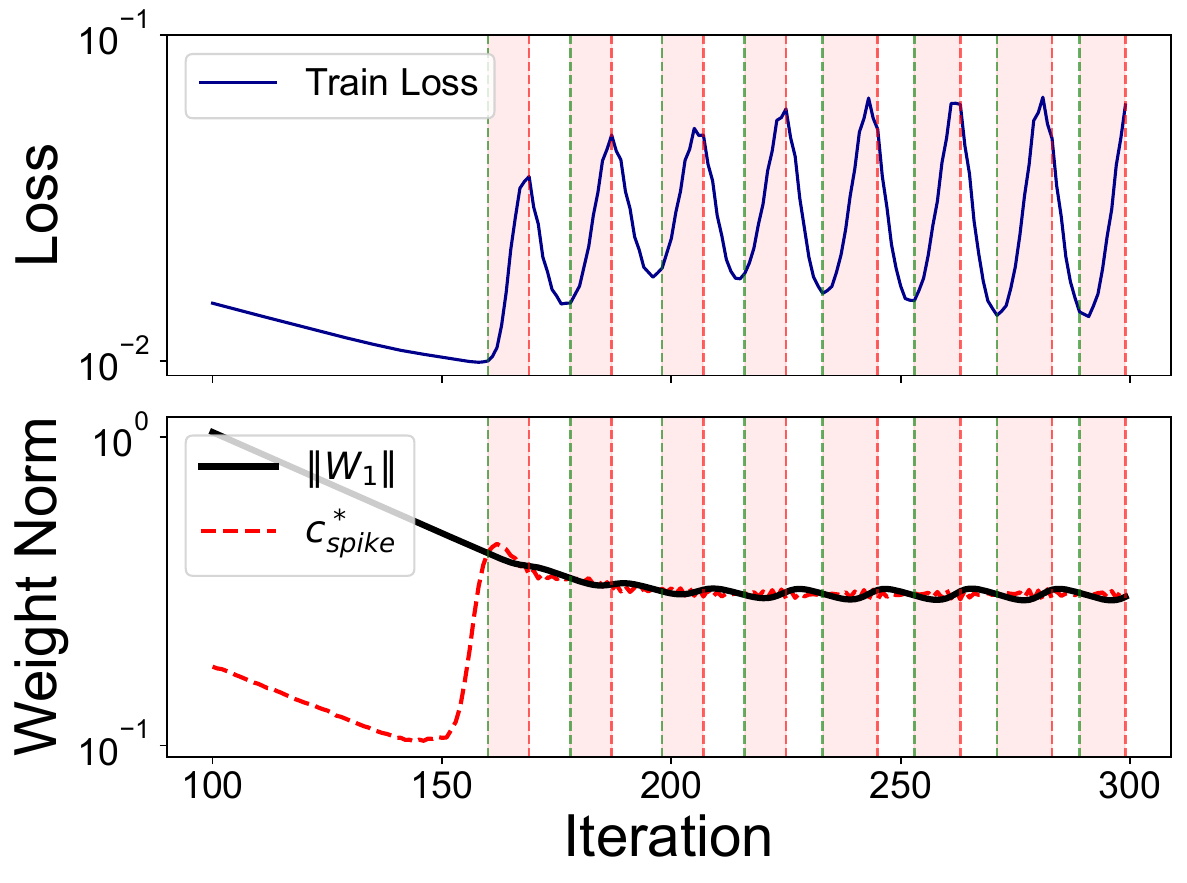}
    \caption{Synthetic data}
    \label{eos_synthetic}
  \end{subfigure}
  \hfill
  \begin{subfigure}{0.51\textwidth}
    \centering
    \includegraphics[width=0.9\linewidth,
    height=0.22\textheight]{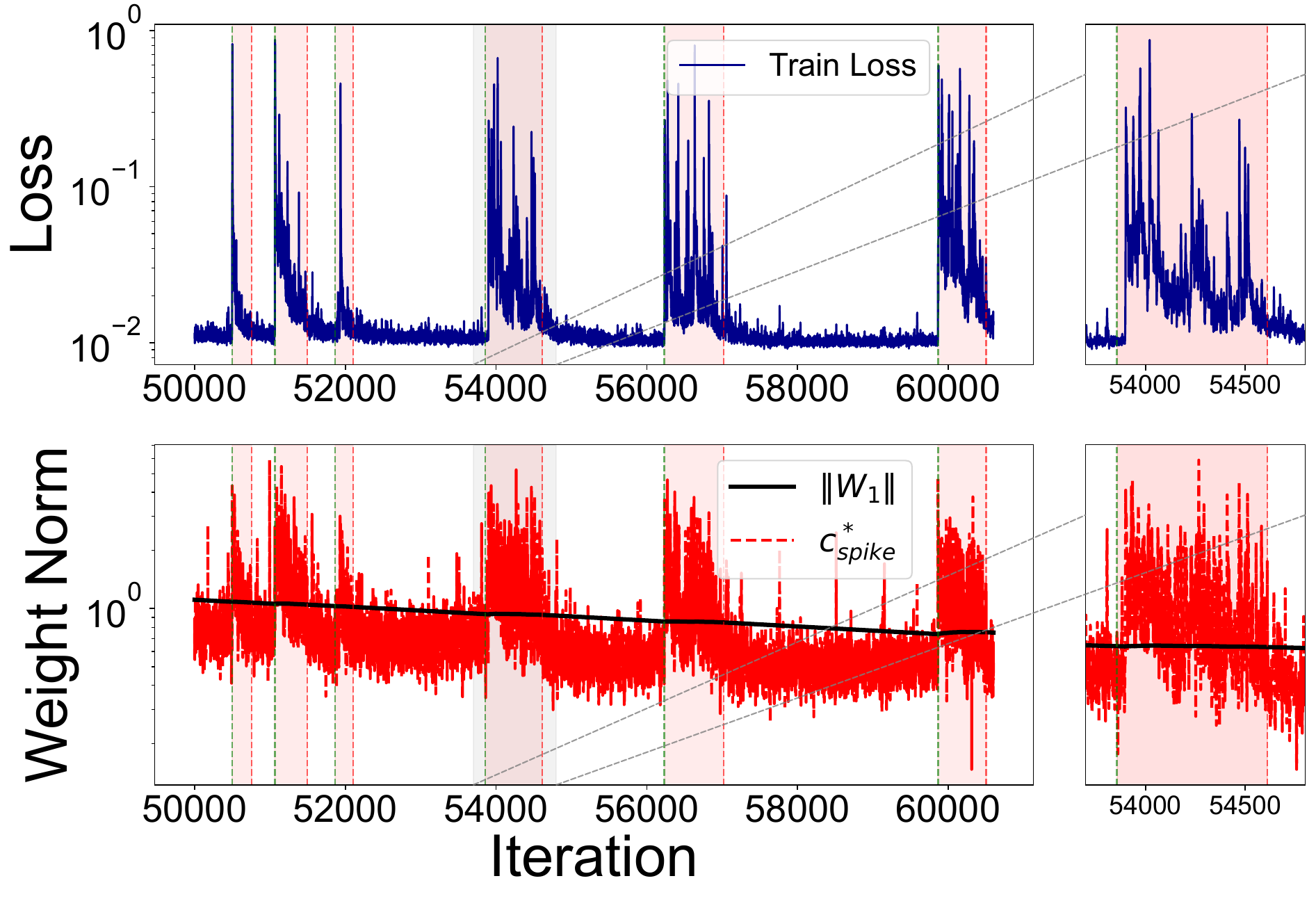}
    \caption{MNIST}
    \label{eos_MNIST}
  \end{subfigure}
  \caption{
  \textbf{Weight-norm criticality and instability under BN.}
  Training dynamics of FNN with BN.
  In each subfigure, the top panel shows the training loss trajectory during the later stages of optimization,
  while the bottom panel plots the weight norm $\|W_1\|$ of the first (scale-invariant) layer (black)
  together with the stability boundary $c_{spike}^*$ (red).
  Red shaded regions indicate intervals where $\|W_1\|$ persistently falls below $c_{spike}^*$,
  corresponding to the theoretically predicted unstable regime.
  Pronounced loss spikes align with sustained excursions into this region.
  }
  \label{fig:eos_weight_norm_criticality}
\end{figure*}

Theorem~\ref{thm:curvature_explosion} predicts that, under positive scale invariance, shrinking the scale-invariant parameters $u$ by a factor $\alpha$ amplifies the curvature in the $u$-subspace on the order of $\alpha^{-2}$, which in turn induces a lower bound on the growth of $\lambda_{\max}\big(H(\alpha u,v)\big)$ via Eq.~\eqref{eq:curvature_lb_split}. 
We empirically test this prediction by performing controlled parameter scaling and measuring $\lambda_{\max}$ as a function of the scaling factor.

For ResNet-50 trained on CIFAR-100, we rescale the scale-invariant parameters while keeping the remaining parameters fixed. As the scaling factor decreases, $\lambda_{\max}$ increases rapidly, consistent with the theoretical lower bound (Fig.~\ref{fig:resnet_scale}).



The evolution of $\lambda_{\max}$ of the Hessian blocks associated with the scale-invariant and non-scale-invariant components, together with that of the full Hessian, is provided in Fig.~\ref{hessian_trajectory_with_zoom}. At initialization (epoch $=0$), the maximum eigenvalue of the scale-invariant block (orange) is smaller than that of the non-scale-invariant block (green), while both remain within the same order of magnitude (see the zoomed-in panel in Fig.~\ref{hessian_trajectory_with_zoom}). As training proceeds, the scale-invariant block's maximum eigenvalue increases rapidly and soon surpasses the non-scale-invariant one. By the mid-to-late stages of training, it far exceeds the non-scale-invariant block and becomes the dominant contribution to the global $\lambda_{\max}$ (blue), closely aligning with the scaling behavior implied by the scale-invariant Hessian analysis (Fig.~\ref{hessian_trajectory_with_zoom}).



\subsection{Weight-norm Criticality}
\label{Weight-norm Criticality}

\

Recent empirical evidence suggests that modern deep networks are often trained near the \emph{EoS}~\citep{wu2018sgd, cohen2021gradient}, a regime in which the local curvature (sharpness) of the loss landscape approaches the maximum threshold for gradient descent (GD) to remain stable.
For GD with learning rate $\eta$, the standard EoS stability criterion can be expressed as
\begin{equation}
\label{eq:gd_linear_stability}
\eta\,\lambda_{\max}\big(H(u,v)\big) \le 2,
\end{equation}
where $H(u,v)=\nabla^2_{(u,v)}L(u,v)$ denotes the Hessian with respect to model parameters. Combining Eq.~\eqref{eq:gd_linear_stability} with the homogeneity/scaling property derived above for scale-invariant parameters, we can translate a curvature threshold into a \emph{weight-norm threshold}.
This motivates the definition of a stability boundary on $\|u\|$ that guarantees linear stability (Proposition~\ref{prop:eos_boundary}).

\begin{proposition}[The Weight-Norm Stability Boundary]
\label{prop:eos_boundary}
Let $\rho := \|u\|^2\cdot\lambda_{\max}\big(H_{uu}(u, v)\big)$ be the intrinsic curvature of the loss landscape. The boundary of weight norm $c^*$ required to maintain linear stability is given by:
\small{\begin{equation}
\label{eq:boundary_formula}
    c^* = \sqrt{\frac{\eta\rho}{2}}.
\end{equation}}
\end{proposition}

\begin{proof}
Let the current scale-invariant component parameter be $w$. By Theorem \ref{thm:curvature_explosion}, the current sharpness 
\begin{equation}
\begin{aligned}
\lambda_{\max}\big(H(w, v)\big) & \geq \lambda_{\max}\big(H_{uu}(w, v)\big) \\
&=\frac{\|u\|^2}{\|w\|^2}\lambda_{\max}\big(H_{uu}(u, v)\big) \\
&= \frac{\rho}{\|w\|^2}. 
\end{aligned}
\end{equation}
The stability conditions (Eq.~\ref{eq:gd_linear_stability}) requires:
\begin{equation}
    \eta \cdot \lambda_{\max}\big(H(w, v)\big) \le 2 \implies \eta \cdot \frac{\rho}{\|w\|^2} \le 2.
\end{equation}
Solving for $\|w\|$ yields the lower bound $\|w\| \ge \sqrt{\frac{\eta \rho}{2}}$. We define the equality case as the weight-norm stability boundary $c^*$.
\end{proof}

\begin{figure*}[t]
  \centering
\includegraphics[width=0.9\textwidth,
    height=0.28\textheight]{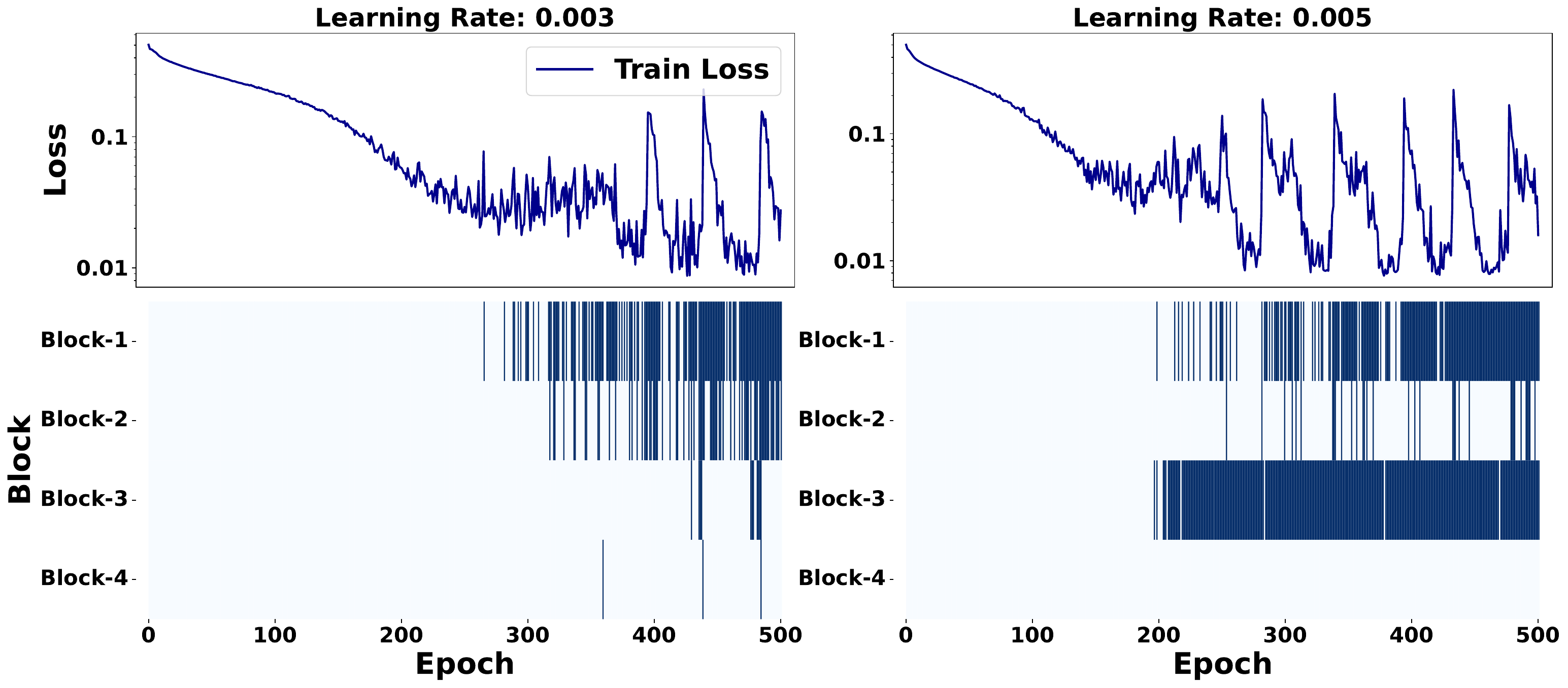}
    \caption{\textbf{Block-wise validation of the weight-norm criticality $c_{spike}^*$ in ResNet-50 trained on CIFAR-100.} The heatmap shows whether the weight norm of each selected layer crosses its corresponding critical boundary, revealing clear heterogeneity in the tendency of different layers to reach weight-norm criticality.}
  \label{fig:Resnet_spike_boundary_Heatmap}
\end{figure*}

The EoS criterion based on $\lambda_{\max}(H)$ is a \emph{worst-case} condition: it only asserts the existence of a direction with large curvature and characterizes marginal linear stability.
In contrast, whether the loss strictly increases after a single GD step depends on the curvature \emph{along the gradient direction}. This observation was highlighted by \citet{bai2025adaptive}, who formally defined the \emph{gradient-direction curvature} as:
\[
\lambda_{\mathrm{grad}}(H) := \frac{g^\top H g}{\|g\|^2}\qquad(\|g\|>0),
\]
where $g := \nabla \mathcal{L}$. This quantity serves as a primary indicator of potential loss increasing and, more critically, the onset of loss spikes during the optimization process.
A second-order analysis (in the Appendix~D) shows that a one-step \emph{loss spike} is triggered when
\[
\eta\,\lambda_{\mathrm{grad}}\big(H(u,v)\big) > 2,
\]
and, under the same scale-invariant homogeneity, this condition induces an analogous critical norm threshold.
This motivates the \emph{spike boundary} based on $\lambda_{\mathrm{grad}}$ (Definition~\ref{def:spike_boundary}).

\begin{proposition}[The Weight-Norm Spike Boundary]
\label{def:spike_boundary}
Let $g_u := \nabla_u \mathcal{L}(u,v)$ and define the gradient-direction curvature
\begin{equation}
\label{eq:lambda_grad_def}
\lambda_{\mathrm{grad}}\big(H_{uu}(u,v)\big):=
\dfrac{g_u^\top H_{uu}(u,v)\, g_u}{\|g_u\|^2}.
\end{equation}
Define the intrinsic spike curvature of the loss landscape as
\begin{equation}
\label{eq:rho_grad_def}
\rho_{\mathrm{grad}} := \|u\|^2 \cdot \lambda_{\mathrm{grad}}\big(H_{uu}(u,v)\big).
\end{equation}
Then the weight-norm spike boundary is defined by
\begin{equation}
\label{eq:spike_boundary_formula}
c_{\mathrm{spike}}^* := \sqrt{\frac{\eta\,\rho_{\mathrm{grad}}}{2}}.
\end{equation}
\end{proposition}

\section{Empirical Validation of Weight-Norm Criticality}

\

We validate the predicted $c_{\mathrm{spike}}^*$ in FNN with normalization on both the MNIST classification task and the synthetic regression task (detail in Appendix Fig.~A.1). Our theory provides a layer-wise stability boundary. Accordingly, Fig.~\ref{eos_synthetic} and \ref{eos_MNIST} report the first-layer dynamics; similar agreement between the predicted weight-norm boundary and optimization instability is consistently observed in the other scale-invariant layers; the corresponding results are deferred to Appendix~C. We further repeat these experiments across multiple runs with different random seeds, consistently observing the same behavior; detailed results are provided in Appendix~E.

Building on prior empirical findings that temporarily exceeding the stability threshold does not necessarily trigger a pronounced loss spike \citep{li2023loss,bai2025adaptive}, we apply a filtering protocol to the identified unstable intervals. This procedure effectively suppresses brief, non-macroscopic excursions that do not reflect sustained instability.

Across both experiments, excursions into the predicted unstable regime—i.e., when the weight norm drops below the stability boundary—align with pronounced instabilities in the optimization trajectory, including abrupt loss spikes. To make this correspondence explicit, we mark the onset of instability (first crossing below the boundary) with green vertical dashed lines and the return to stability with red vertical dashed lines. 

\begin{figure*}[t]
\centering
\includegraphics[width = 0.98\textwidth,
    height=0.28\textheight]{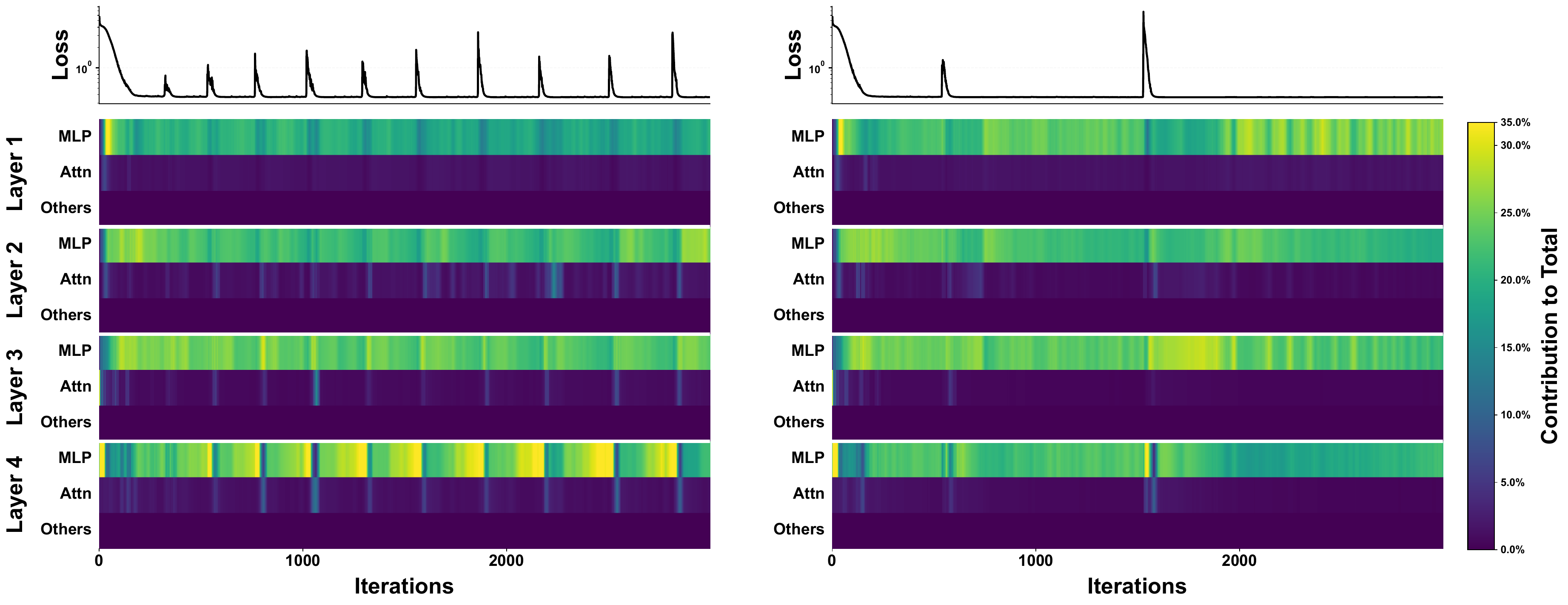}
    \caption{\textbf{Evolution of the module-wise decomposition of the top Hessian eigenvector} during training Transformer on the $3x \to x$ task, following \citet{zhang2024anchorfunctiontypebenchmark}; see Appendix~C for full details. \textbf{Left:} Apply weight decay to all modules; \textbf{Right:} Disable weight decay for the MLP modules. Different colors correspond to different modules.}
    \label{fig:transformer_module_top_eig_frac}
\end{figure*}

We further validate the proposed spike boundary $c_{\mathrm{spike}}^*$ on ResNet-50 trained on CIFAR-100. Specifically, we examine four representative blocks of ResNet-50 under different learning rates and track their corresponding weight-norm criticality. The results reveal clear layer-wise heterogeneity: different layers exhibit distinct tendencies to approach or cross their critical weight-norm boundaries. To visualize this behavior, we construct heatmaps indicating whether each block enters the predicted unstable regime during training (Fig.~\ref{fig:Resnet_spike_boundary_Heatmap}). These results demonstrate that crossing a block-wise weight-norm criticality drives the model into an unstable regime and may ultimately trigger a loss spike.


\section{Training Instability in Large Language Models and Module-wise Dynamics}

\subsection{Training Instability in LLMs}

\

Regularization plays an important role in training large Transformer models, with weight decay being widely used in practice. For parameters associated with scale-invariant components, rescaling their norm leaves the normalized activations, and thus the network output, nearly unchanged. Consequently, weight decay can continuously shrink these parameters without causing a comparable change in the represented function. This makes the effect of weight decay particularly pronounced on scale-invariant parameters, rather than directly altering the represented function.

This decoupling between parameter norm and model output can gradually affect gradient magnitudes and local curvature during training. Although practical Transformer architectures are not strictly scale-invariant, we observe a consistent empirical trend across Pre-LayerNorm, Post-LayerNorm, Sandwich-Norm, RMSNorm, and DeepNorm: stronger weight decay leads to more frequent loss spikes, as detailed in Appendix~F. This indicates that the interaction between normalization and weight decay remains relevant beyond the strictly scale-invariant setting of our theory. In deep Transformers, these effects may accumulate across layers and contribute to increased sharpness and training instability under strong weight decay, as illustrated in Fig.~\ref{fig:transformer_100b}. This motivates the analysis that follows.

\subsection{Module-wise Hessian Eigenvector Analysis}

\

To identify where the dominant training curvature concentrates, we track the top Hessian eigenvector and decompose its squared norm across parameter modules (see Appendix~C for details). For each module, we measure its fraction of the squared norm of the top eigenvector.
Fig.~\ref{fig:transformer_module_top_eig_frac}a shows that the Multilayer Perceptron (MLP) blocks dominate the leading-curvature direction, with their contribution increasing during training and changing markedly around loss spikes. This suggests that the instability associated with weight decay is amplified in the MLP subspace, particularly in Layer~4 (Fig.~\ref{fig:transformer_module_top_eig_frac}(Left)). Motivated by this observation, we disable weight decay for the MLP parameters while keeping all other settings unchanged. As shown in Fig.~\ref{fig:transformer_module_top_eig_frac}(Right), the increasing MLP contribution is largely suppressed, and the number of loss spikes is substantially reduced.

\begin{figure}[h]
  \centering
  \includegraphics[
    height=0.18\textheight]{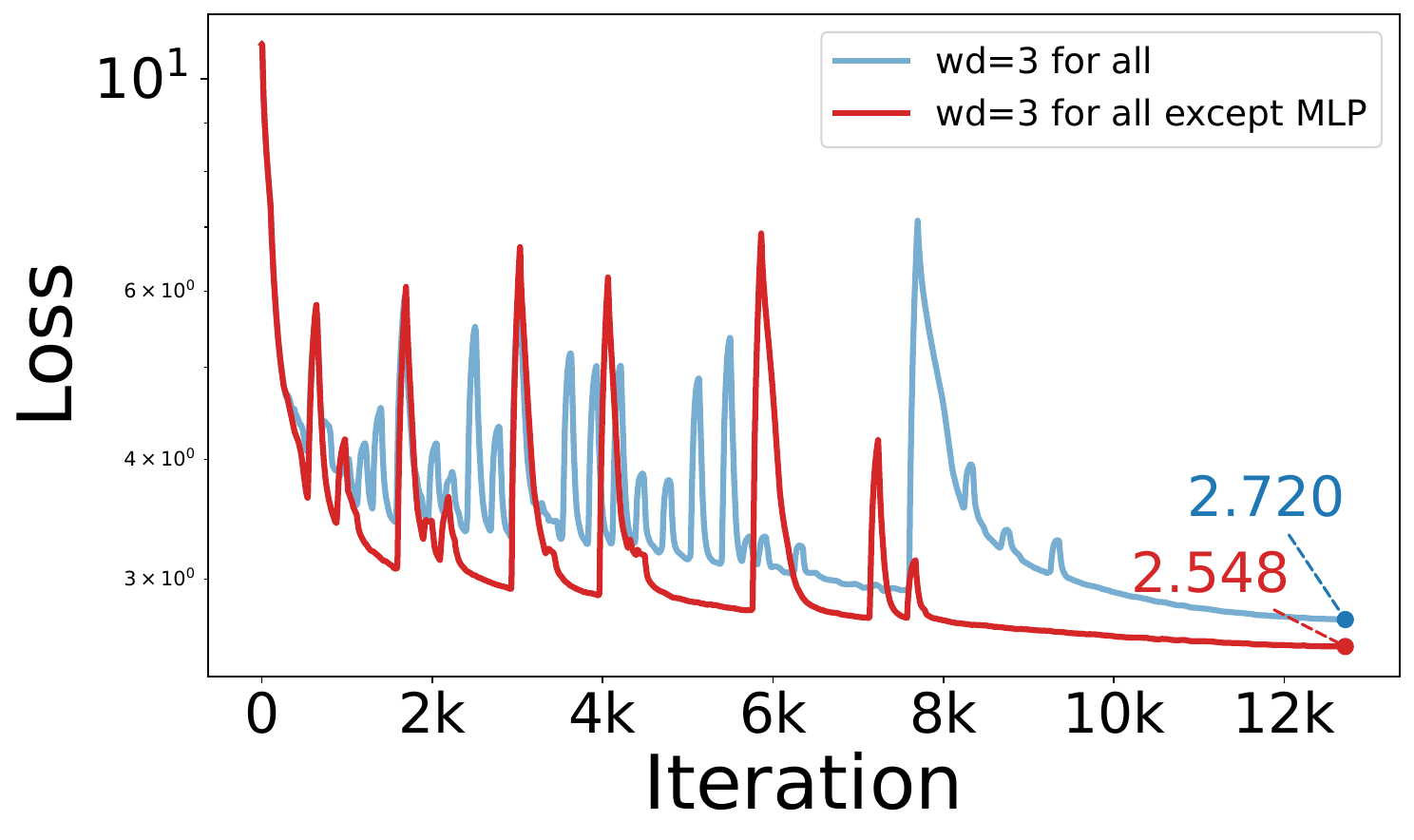}
  \caption{Evolution of training loss for a 187M-parameter Transformer, comparing standard weight decay against a configuration where weight decay is excluded from the MLP modules in each layer.}
  \label{GPT_wd_3}
\end{figure}

We further validate these observations at a larger scale. As shown in Fig.~\ref{GPT_wd_3}, when weight decay is applied to all parameters, the training dynamics become highly unstable (blue), exhibiting frequent and pronounced loss spikes. In contrast, when weight decay is disabled for the MLP parameters (red), the frequency of loss spikes is reduced, and the training loss consistently attains lower values.

These results indicate that a mechanistic understanding of loss spikes can inform practical training strategies for large-scale models. Beyond mitigating training instability, such insights may also enable more effective regularization choices that improve generalization performance.

\section{Discussion}

\

This work identifies an additional loss-spike phenomenon that is frequently encountered in practice induced by the interaction between normalization and weight decay. While existing explanations commonly attribute loss spikes to factors such as data, loss landscape geometry, or optimization—often unified through the notion of learning-rate criticality. We provide a new mechanistic understanding of loss spikes through the lens of weight-norm criticality.

Weight decay on scale-invariant components increases curvature in the corresponding parameter subspaces, making the loss landscape increasingly sharp. Building on this observation, we provide analysis of how the decay of scale-invariant parameter norms influences the overall sharpness of the loss landscape. We quantify this effect by defining a weight-norm criticality for each scale-invariant component. Concretely, when the parameter norm of a scale-invariant component crosses the stability boundary—or, in more severe cases, the spike boundary—the training dynamics become unstable and may exhibit loss spikes. Importantly, our analysis does not rely on global scale-invariance assumptions, and therefore applies to practical deep neural networks that include normalization layers such as LN. Furthermore, component-wise criticality identifies the sources of loss spikes and predicts impending instabilities. We empirically validate these theoretical boundaries on Transformer and ResNet-50, where the observed behavior closely matches the theoretical predictions.

The mechanisms underlying loss spikes in general settings remain incompletely understood. Meanwhile, the implications of loss spikes for generalization are also not yet well understood. In contemporary large-scale model training, loss spikes are typically regarded as undesirable instabilities and are therefore avoided in practice, even though they may in some cases coincide with improved test performance. This work offers a new perspective by highlighting the often overlooked weight-norm criticality. We hope that this analysis can serve as a basis for further investigation into training instability in neural networks.

\bibliography{aaai2027}

@article{van2017l2,
  title={L2 regularization versus batch and weight normalization},
  author={Van Laarhoven, Twan},
  journal={arXiv preprint arXiv:1706.05350},
  year={2017}
}

@inproceedings{
zhang2018three,
title={Three Mechanisms of Weight Decay Regularization},
author={Guodong Zhang and Chaoqi Wang and Bowen Xu and Roger Grosse},
booktitle={International Conference on Learning Representations},
year={2019},
url={https://openreview.net/forum?id=B1lz-3Rct7},
}

@inproceedings{NEURIPS2018_a0160709,
 author = {Hoffer, Elad and Banner, Ron and Golan, Itay and Soudry, Daniel},
 booktitle = {Advances in Neural Information Processing Systems},
 editor = {S. Bengio and H. Wallach and H. Larochelle and K. Grauman and N. Cesa-Bianchi and R. Garnett},
 pages = {},
 publisher = {Curran Associates, Inc.},
 title = {Norm matters: efficient and accurate normalization schemes in deep networks},
 volume = {31},
 year = {2018}
}

@inproceedings{NEURIPS2020_a7453a5f,
 author = {Li, Zhiyuan and Lyu, Kaifeng and Arora, Sanjeev},
 booktitle = {Advances in Neural Information Processing Systems},
 editor = {H. Larochelle and M. Ranzato and R. Hadsell and M.F. Balcan and H. Lin},
 pages = {14544--14555},
 publisher = {Curran Associates, Inc.},
 title = {Reconciling Modern Deep Learning with Traditional Optimization Analyses: The Intrinsic Learning Rate},
 volume = {33},
 year = {2020}
}

@inproceedings{NEURIPS2022_dffd1c52,
 author = {Lyu, Kaifeng and Li, Zhiyuan and Arora, Sanjeev},
 booktitle = {Advances in Neural Information Processing Systems},
 editor = {S. Koyejo and S. Mohamed and A. Agarwal and D. Belgrave and K. Cho and A. Oh},
 pages = {34689--34708},
 publisher = {Curran Associates, Inc.},
 title = {Understanding the Generalization Benefit of Normalization Layers: Sharpness Reduction},
 volume = {35},
 year = {2022}
}

@inproceedings{NEURIPS2019_cb3ce9b0,
 author = {Chiley, Vitaliy and Sharapov, Ilya and Kosson, Atli and Koster, Urs and Reece, Ryan and Samaniego de la Fuente, Sofia and Subbiah, Vishal and James, Michael},
 booktitle = {Advances in Neural Information Processing Systems},
 editor = {H. Wallach and H. Larochelle and A. Beygelzimer and F. d\textquotesingle Alch\'{e}-Buc and E. Fox and R. Garnett},
 pages = {},
 publisher = {Curran Associates, Inc.},
 title = {Online Normalization for Training Neural Networks},
 volume = {32},
 year = {2019}
}

@inproceedings{NEURIPS2021_326a8c05,
 author = {Wan, Ruosi and Zhu, Zhanxing and Zhang, Xiangyu and Sun, Jian},
 booktitle = {Advances in Neural Information Processing Systems},
 editor = {M. Ranzato and A. Beygelzimer and Y. Dauphin and P.S. Liang and J. Wortman Vaughan},
 pages = {6380--6391},
 publisher = {Curran Associates, Inc.},
 title = {Spherical Motion Dynamics: Learning Dynamics of Normalized Neural Network using SGD and Weight Decay},
 volume = {34},
 year = {2021}
}

@article{li2019exponential,
  title={An exponential learning rate schedule for deep learning},
  author={Li, Zhiyuan and Arora, Sanjeev},
  journal={arXiv preprint arXiv:1910.07454},
  year={2019}
}

@article{Li_Chen_Yang_2020, title={Understanding the Disharmony between Weight Normalization Family and Weight Decay}, volume={34}, url={https://ojs.aaai.org/index.php/AAAI/article/view/5904}, DOI={10.1609/aaai.v34i04.5904}, number={04}, journal={Proceedings of the AAAI Conference on Artificial Intelligence}, author={Li, Xiang and Chen, Shuo and Yang, Jian}, year={2020}, month={Apr.}, pages={4715-4722} }

@inproceedings{NEURIPS2021_b433da1b,
 author = {Lobacheva, Ekaterina and Kodryan, Maxim and Chirkova, Nadezhda and Malinin, Andrey and Vetrov, Dmitry P},
 booktitle = {Advances in Neural Information Processing Systems},
 editor = {M. Ranzato and A. Beygelzimer and Y. Dauphin and P.S. Liang and J. Wortman Vaughan},
 pages = {21545--21556},
 publisher = {Curran Associates, Inc.},
 title = {On the Periodic Behavior of Neural Network Training with Batch Normalization and Weight Decay},
 volume = {34},
 year = {2021}
}

@inproceedings{NEURIPS2022_5aea56ee,
 author = {Kodryan, Maxim and Lobacheva, Ekaterina and Nakhodnov, Maksim and Vetrov, Dmitry P},
 booktitle = {Advances in Neural Information Processing Systems},
 editor = {S. Koyejo and S. Mohamed and A. Agarwal and D. Belgrave and K. Cho and A. Oh},
 pages = {14058--14070},
 publisher = {Curran Associates, Inc.},
 title = {Training Scale-Invariant Neural Networks on the Sphere Can Happen in Three Regimes},
 volume = {35},
 year = {2022}
}

@article{lewkowycz2020large,
  title={The large learning rate phase of deep learning: the catapult mechanism},
  author={Lewkowycz, Aitor and Bahri, Yasaman and Dyer, Ethan and Sohl-Dickstein, Jascha and Gur-Ari, Guy},
  journal={arXiv preprint arXiv:2003.02218},
  year={2020}
}

@article{chowdhery2023palm,
  title={Palm: Scaling language modeling with pathways},
  author={Chowdhery, Aakanksha and Narang, Sharan and Devlin, Jacob and Bosma, Maarten and Mishra, Gaurav and Roberts, Adam and Barham, Paul and Chung, Hyung Won and Sutton, Charles and Gehrmann, Sebastian and others},
  journal={Journal of Machine Learning Research},
  volume={24},
  number={240},
  pages={1--113},
  year={2023}
}

@inproceedings{zhang2024anchorfunctiontypebenchmark,
      title={Anchor function: a type of benchmark functions for studying language models}, 
      author={Zhongwang Zhang and Zhiwei Wang and Junjie Yao and Zhangchen Zhou and Xiaolong Li and Weinan E and Zhi-Qin John Xu},
      year={2025},
      url={https://arxiv.org/abs/2401.08309}, 
    booktitle={ICLR 2025 Workshop Bridging the Gap Between Practice and Theory in Deep Learning}
}

@inproceedings{
cohen2021gradient,
title={Gradient Descent on Neural Networks Typically Occurs at the Edge of Stability},
author={Jeremy Cohen and Simran Kaur and Yuanzhi Li and J Zico Kolter and Ameet Talwalkar},
booktitle={International Conference on Learning Representations},
year={2021},
url={https://openreview.net/forum?id=jh-rTtvkGeM}
}

@article{wu2018sgd,
  title={How sgd selects the global minima in over-parameterized learning: A dynamical stability perspective},
  author={Wu, Lei and Ma, Chao and E, Weinan},
  journal={Advances in Neural Information Processing Systems},
  volume={31},
  year={2018}
}

@inproceedings{ahn2022understanding,
  title={Understanding the unstable convergence of gradient descent},
  author={Ahn, Kwangjun and Zhang, Jingzhao and Sra, Suvrit},
  booktitle={International conference on machine learning},
  pages={247--257},
  year={2022},
  organization={PMLR}
}

@article{xing2018walk,
  title={A walk with sgd},
  author={Xing, Chen and Arpit, Devansh and Tsirigotis, Christos and Bengio, Yoshua},
  journal={arXiv preprint arXiv:1802.08770},
  year={2018}
}

@article{wang2022analyzing,
  title={Analyzing sharpness along GD trajectory: Progressive sharpening and edge of stability},
  author={Wang, Zixuan and Li, Zhouzi and Li, Jian},
  journal={Advances in Neural Information Processing Systems},
  volume={35},
  pages={9983--9994},
  year={2022}
}

@inproceedings{
damian2023selfstabilization,
title={Self-Stabilization: The Implicit Bias of Gradient Descent at the Edge of Stability},
author={Alex Damian and Eshaan Nichani and Jason D. Lee},
booktitle={The Eleventh International Conference on Learning Representations },
year={2023},
url={https://openreview.net/forum?id=nhKHA59gXz}
}

@inproceedings{
Jastrzebski2020The,
title={The Break-Even Point on Optimization Trajectories of Deep Neural Networks},
author={Stanislaw Jastrzebski and Maciej Szymczak and Stanislav Fort and Devansh Arpit and Jacek Tabor and Kyunghyun Cho* and Krzysztof Geras*},
booktitle={International Conference on Learning Representations},
year={2020},
url={https://openreview.net/forum?id=r1g87C4KwB}
}

@inproceedings{arora2022understanding,
  title={Understanding gradient descent on the edge of stability in deep learning},
  author={Arora, Sanjeev and Li, Zhiyuan and Panigrahi, Abhishek},
  booktitle={International Conference on Machine Learning},
  pages={948--1024},
  year={2022},
  organization={PMLR}
}

@inproceedings{ma2022qualitative,
  title={A qualitative study of the dynamic behavior for adaptive gradient algorithms},
  author={Ma, Chao and Wu, Lei and E, Weinan},
  booktitle={Mathematical and scientific machine learning},
  pages={671--692},
  year={2022},
  organization={PMLR}
}

@article{li2023loss,
  title={Loss spike in training neural networks},
  author={Li, Xiaolong and Xu, Zhi-Qin John and Zhang, Zhongwang},
  journal={Journal of Computational Mathematics},
  year={2025}
}

@article{molybog2023theory,
  title={A theory on adam instability in large-scale machine learning},
  author={Molybog, Igor and Albert, Peter and Chen, Moya and DeVito, Zachary and Esiobu, David and Goyal, Naman and Koura, Punit Singh and Narang, Sharan and Poulton, Andrew and Silva, Ruan and others},
  journal={arXiv preprint arXiv:2304.09871},
  year={2023}
}

@inproceedings{
Jastrzebski2018on,
title={On the Relation Between the Sharpest Directions of {DNN} Loss and the {SGD} Step Length},
author={Stanislaw Jastrzebski and Zachary Kenton and Nicolas Ballas and Asja Fischer and Yoshua Bengio and Amost Storkey},
booktitle={International Conference on Learning Representations},
year={2019},
url={https://openreview.net/forum?id=SkgEaj05t7},
}

@article{touvron2023llama,
  title={Llama 2: Open foundation and fine-tuned chat models},
  author={Touvron, Hugo and Martin, Louis and Stone, Kevin and Albert, Peter and Almahairi, Amjad and Babaei, Yasmine and Bashlykov, Nikolay and Batra, Soumya and Bhargava, Prajjwal and Bhosale, Shruti and others},
  journal={arXiv preprint arXiv:2307.09288},
  year={2023}
}

@article{bai2025adaptive,
  title={Adaptive Preconditioners Trigger Loss Spikes in Adam},
  author={Bai, Zhiwei and Zhou, Zhangchen and Zhao, Jiajie and Li, Xiaolong and Li, Zhiyu and Xiong, Feiyu and Yang, Hongkang and Zhang, Yaoyu and Xu, Zhi-Qin John},
  journal={arXiv preprint arXiv:2506.04805},
  year={2025}
}

@inproceedings{ioffe2015batch,
  title={Batch normalization: Accelerating deep network training by reducing internal covariate shift},
  author={Ioffe, Sergey and Szegedy, Christian},
  booktitle={International conference on machine learning},
  pages={448--456},
  year={2015},
  organization={pmlr}
}

@article{ba2016layer,
  title={Layer normalization},
  author={Ba, Jimmy Lei and Kiros, Jamie Ryan and Hinton, Geoffrey E},
  journal={arXiv preprint arXiv:1607.06450},
  year={2016}
}

@article{fisk2005very,
  title={A very short proof of Cauchy's interlace theorem for eigenvalues of Hermitian matrices},
  author={Fisk, Steve},
  journal={arXiv preprint math/0502408},
  year={2005}
}

@inproceedings{he2016deep,
  title={Deep residual learning for image recognition},
  author={He, Kaiming and Zhang, Xiangyu and Ren, Shaoqing and Sun, Jian},
  booktitle={Proceedings of the IEEE conference on computer vision and pattern recognition},
  pages={770--778},
  year={2016}
}

@article{vaswani2017attention,
  title={Attention is all you need},
  author={Vaswani, Ashish and Shazeer, Noam and Parmar, Niki and Uszkoreit, Jakob and Jones, Llion and Gomez, Aidan N and Kaiser, {\L}ukasz and Polosukhin, Illia},
  journal={Advances in neural information processing systems},
  volume={30},
  year={2017}
}

\appendix
\onecolumn

\renewcommand{\thesection}{\Alph{section}}

\counterwithin{figure}{section}
\renewcommand{\thefigure}{\thesection.\arabic{figure}}

\begin{center}
    \noindent\rule{\linewidth}{0.6pt}\par
    \vspace{0.7em}
    
    {\LARGE\bfseries Appendix\par}
    
    \vspace{0.5em}
    \noindent\rule{\linewidth}{0.4pt}\par
\end{center}

\vspace{0.5em}

\noindent
This appendix provides supplementary material to complement
the main text, including implementation details, additional
experimental results, and further analysis.

\section{Proof of Theorem 1 (Curvature Explosion Induced by Scale Invariance)}
\label{app:curvature_explosion_proof}

\begin{lemma}[Hessian Homogeneity under Scale Invariant Parameters]
\label{thm:hessian_homogeneity}
Let $L(u, v)$ be twice continuously differentiable and satisfy positive scale invariance with respect to $u$. Let $H(u,v)$ be partitioned into blocks $H_{uu}, H_{uv}, H_{vu}, H_{vv}$. For any scaling factor $\alpha > 0$, the Hessian at the scaled point $(\alpha u, v)$ satisfies the following rescaling property:
{\small\begin{equation}
\label{eq:H_scaling_app_thm}
H(\alpha u,v)
=
\begin{bmatrix}
\alpha^{-2} H_{uu}(u,v) & \alpha^{-1} H_{uv}(u,v) \\
\alpha^{-1} H_{vu}(u,v) & H_{vv}(u,v)
\end{bmatrix}.
\end{equation}}
\end{lemma}

\begin{proof}
Differentiating the invariance condition $L(\alpha u, v) = L(u, v)$ with respect to $u$ yields the gradient homogeneity $\nabla_u L(\alpha u, v) = \frac{1}{\alpha} \nabla_u L(u, v)$. 
Differentiating this relation again with respect to $u$ involves the chain rule (the inner derivative of $\alpha u$ w.r.t $u$ is $\alpha$), leading to:

{\small\begin{equation}
\begin{aligned}
\nabla^2_{uu} L(\alpha u,v)\cdot \alpha
&= \frac{1}{\alpha}\,\nabla^2_{uu} L(u,v),\\
\implies\quad
H_{uu}(\alpha u,v)
&= \frac{1}{\alpha^2}\,H_{uu}(u,v).
\end{aligned}
\end{equation}}

Similarly, differentiating $\nabla_u L(\alpha u, v)$ w.r.t $v$ yields $H_{uv}(\alpha u, v) = \frac{1}{\alpha} H_{uv}(u, v)$. Since $\nabla_v L(\alpha u, v) = \nabla_v L(u, v)$ is invariant, its derivative w.r.t $v$, $H_{vv}$, remains invariant.
\end{proof}

Lemma~\ref{thm:hessian_homogeneity} implies that the Hessian block $H_{uu}$ is homogeneous of degree $-2$. We now quantify the impact of this property on optimization stability.

\begin{lemma}[Cauchy Interlacing Theorem]
\label{lemma:cauchy_interlacing}
Let $A\in\mathbb{R}^{n\times n}$ be symmetric with eigenvalues
\[
\lambda_1(A)\le \lambda_2(A)\le \cdots \le \lambda_n(A).
\]
Let $B\in\mathbb{R}^{m\times m}$ be a \emph{principal submatrix} of $A$ with $1\le m<n$,
and let its eigenvalues be
\[
\mu_1(B)\le \mu_2(B)\le \cdots \le \mu_m(B).
\]
Then for every $i=1,\dots,m$,
{\small\begin{equation}
\label{eq:cauchy_interlacing_general}
\lambda_i(A)\ \le\ \mu_i(B)\ \le\ \lambda_{i+n-m}(A).
\end{equation}}
In particular, if $m=n-1$, then for $i=1,\dots,n-1$,
{\small\begin{equation}
\label{eq:cauchy_interlacing_nminus1}
\lambda_i(A)\ \le\ \mu_i(B)\ \le\ \lambda_{i+1}(A).
\end{equation}}
\end{lemma}

\begin{proof}
For proof details see~\citep{fisk2005very}.
\end{proof}

We now prove Theorem 1 (Curvature Explosion Induced by Scale Invariance).

\begin{proof}
Fix any $\alpha>0$.
By Lemma~\ref{thm:hessian_homogeneity}, the Hessian at the scaled point $(\alpha u,v)$
admits the block rescaling
{\small\begin{equation}
H(\alpha u,v)
=
\begin{bmatrix}
\alpha^{-2} H_{uu}(u,v) & \alpha^{-1} H_{uv}(u,v) \\
\alpha^{-1} H_{vu}(u,v) & H_{vv}(u,v)
\end{bmatrix}.
\end{equation}}
In particular, the $uu$-block satisfies
{\small\begin{equation}
\label{eq:Huu_scaling_thm}
H_{uu}(\alpha u,v)=\alpha^{-2} H_{uu}(u,v).
\end{equation}}

Since $H(\alpha u,v)$ is real symmetric, $H_{uu}(\alpha u,v)$ is a real symmetric
principal submatrix of $H(\alpha u,v)$.
By the Cauchy interlacing theorem (Lemma~\ref{lemma:cauchy_interlacing}),
the largest eigenvalue of the full Hessian is lower bounded by that of any
principal submatrix, yielding
{\small\begin{equation}
\label{eq:interlacing_thm}
\lammax\!\big(H(\alpha u,v)\big)
\ \ge\
\lammax\!\big(H_{uu}(\alpha u,v)\big).
\end{equation}}

Combining \eqref{eq:Huu_scaling_thm} with the homogeneity of eigenvalues for symmetric matrices,
$\lammax(\alpha M)=\alpha\,\lammax(M)$ for $\alpha>0$, we obtain
\[
\lammax\!\big(H_{uu}(\alpha u,v)\big)
=
\lammax\!\big(\alpha^{-2} H_{uu}(u,v)\big)
=
\alpha^{-2}\,\lammax\!\big(H_{uu}(u,v)\big).
\]
Substituting this into \eqref{eq:interlacing_thm} proves
\[
\lammax\!\big(H(\alpha u,v)\big)
\ \ge\
\alpha^{-2}\,\lammax\!\big(H_{uu}(u,v)\big),
\]
which establishes the claimed lower bound.
\end{proof}

\section{Experimental Details}
\label{app:exp_details}

\subsection{Large-scale language model pretraining}
\label{app:llm_pretraining_details}
We use a LLaMA-style Transformer with 16 layers and 16 attention heads.
Each head has dimensionality 80, yielding a hidden size of 1280; the feed-forward hidden size is set equal to the model hidden size.
The model is trained for a single epoch on a 100B-token dataset.
Training uses AdamW with $\beta_1 = 0.9$ and $\beta_2 = 0.999$.
The learning rate follows a linear warmup from zero to a peak of $1\times10^{-3}$ and then decays to a minimum of $1\times10^{-4}$.
We apply global-norm gradient clipping at 1 for stability.
Weight decay is swept over $\{0,\,0.5,\,1\}$, with all other settings held fixed.

\subsection{ResNet-50 on CIFAR-100}
\label{app:resnet_cifar_details}
We train a standard ResNet-50 on CIFAR-100 using SGD with batch size 256 and a fixed learning rate of 0.003.
Weight decay is the only swept hyperparameter; all other settings are held constant across runs.

\subsection{Mechanistic probe on MNIST}
\label{app:mnist_probe_details}

We study image classification on MNIST using a FNN trained with SGD, where normalization can be optionally enabled.
Given an input image $x\in\mathbb{R}^{d}$, the model flattens the input and applies $L=4$ hidden layers of width $h=512$, followed by a linear output layer producing 10 logits.
Each hidden block uses
\[
\texttt{Linear}\;\rightarrow\;\texttt{Norm}\;\rightarrow\;\texttt{Tanh},
\]
where \texttt{Norm} is either BN or LN; when \texttt{norm=none}, the block reduces to $\texttt{Linear}\rightarrow\tanh$.
No normalization is applied to the final output layer.
All models are trained on MNIST with cross-entropy loss using mini-batch SGD (batch size 256).

We consider two groups of experiments:
\emph{(i) Non-normalized baseline.}
In the FNN \emph{without} normalization, weight decay at a fixed level does not reliably induce loss spikes when the learning rate is modest; accordingly, we fix the learning rate to $\eta=0.003$ (Fig.~2a).

\emph{(ii) Normalized (scale-invariant) variants.}
We insert \texttt{BN} or \texttt{LN} into each hidden block to obtain scale-invariant variants and examine how weight decay behaves under this induced scale invariance.
For each normalization choice, we fix the learning rate to $\eta=0.003$ and sweep the weight decay coefficient over
\[
\lambda\in\{0,\,0.001,\,0.01,\,0.03\}.
\]
To ensure a clean comparison when normalization is enabled, we do not apply weight decay to the normalization layers.
Results are shown in Fig.2.

\subsection{Fully controlled regression on synthetic data}
\label{app:minimal_controlled_experiment}

\begin{figure}[h]
  \centering
  \begin{subfigure}[h]{0.49\columnwidth}
    \centering
    \includegraphics[width=\linewidth]{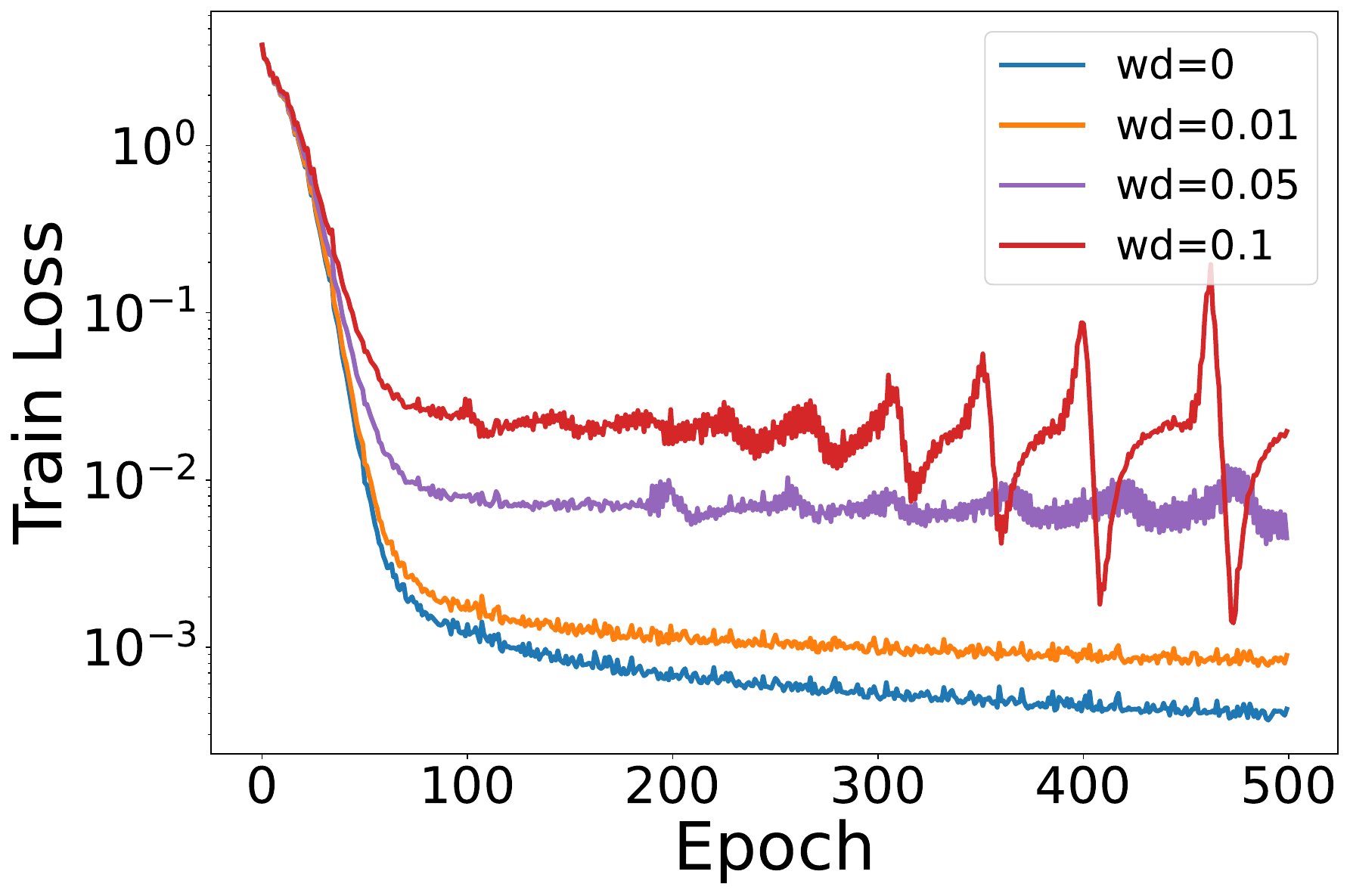}
    \caption{Training loss}
    \label{fig:loss_2x1y}
  \end{subfigure}
  \hfill
  \begin{subfigure}[h]{0.49\columnwidth}
    \centering
    \includegraphics[width=\linewidth]{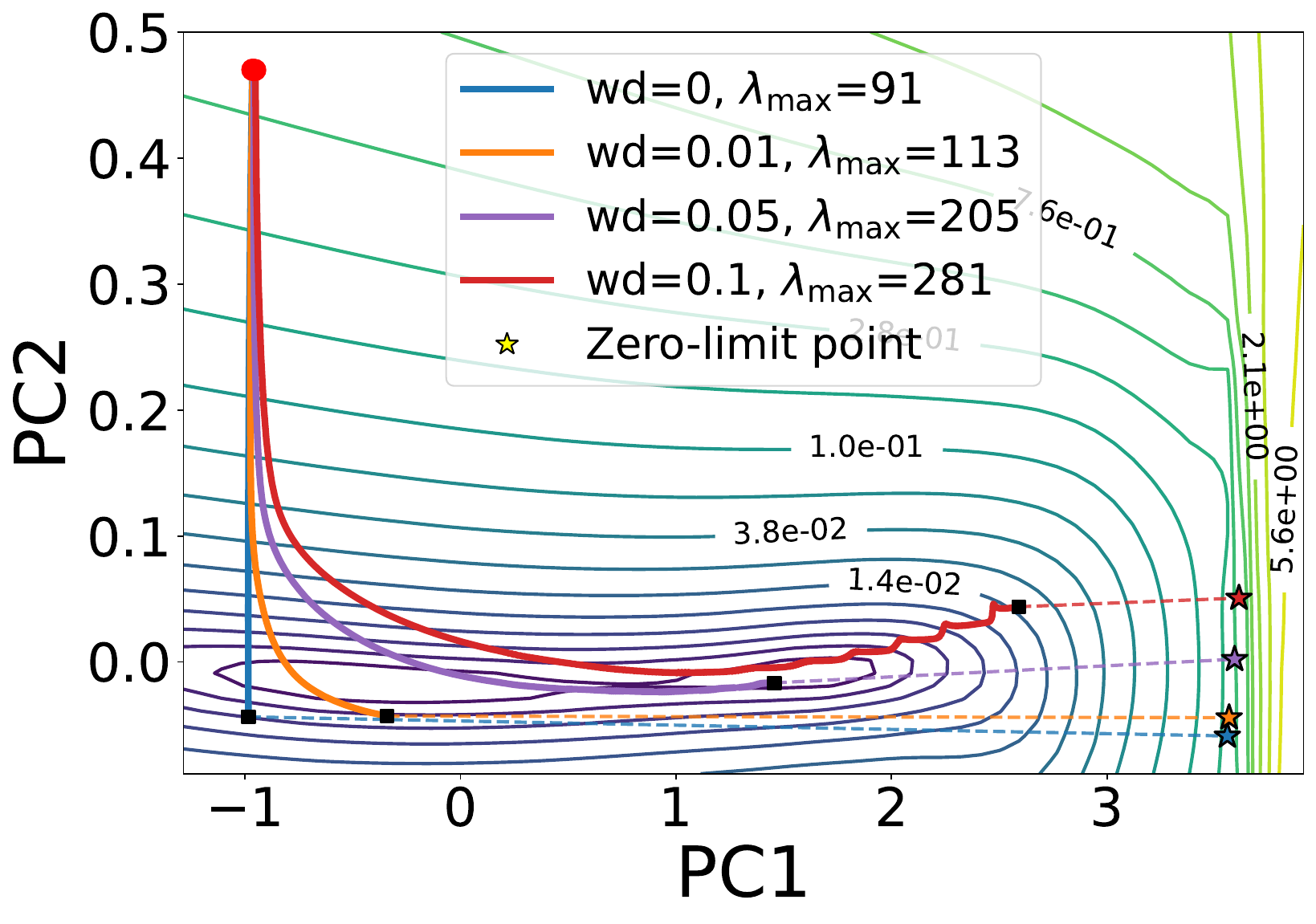}
    \caption{Trajectories}
    \label{fig:2x1y_landscape_traj}
  \end{subfigure}

  \caption{Three-layer FNN with BN on $y=x_1+2x_2$ function. The learning rate is fixed to be $\eta=0.03$.
  (a) \textbf{Training loss trajectories} under a sweep of the weight decay coefficient.
  (b) \textbf{PCA visualization of training trajectories} under varying weight decay. The black squares mark the final parameter states and $\lambda_{\max}$ denotes the largest eigenvalue of the Hessian matrix at the training endpoint. The colored stars represent the derived parameters obtained by freezing the output layer of the trained models and pushing the scale-invariant parameters to zero. Contour lines represent training loss. The explained variance ratios of PC1 and PC2 are $0.9958$ and $0.0020$, respectively.}
  \label{fig:lr003_combined}
\end{figure}


\begin{figure}[h]
  \centering
    \includegraphics[width=0.5\columnwidth]{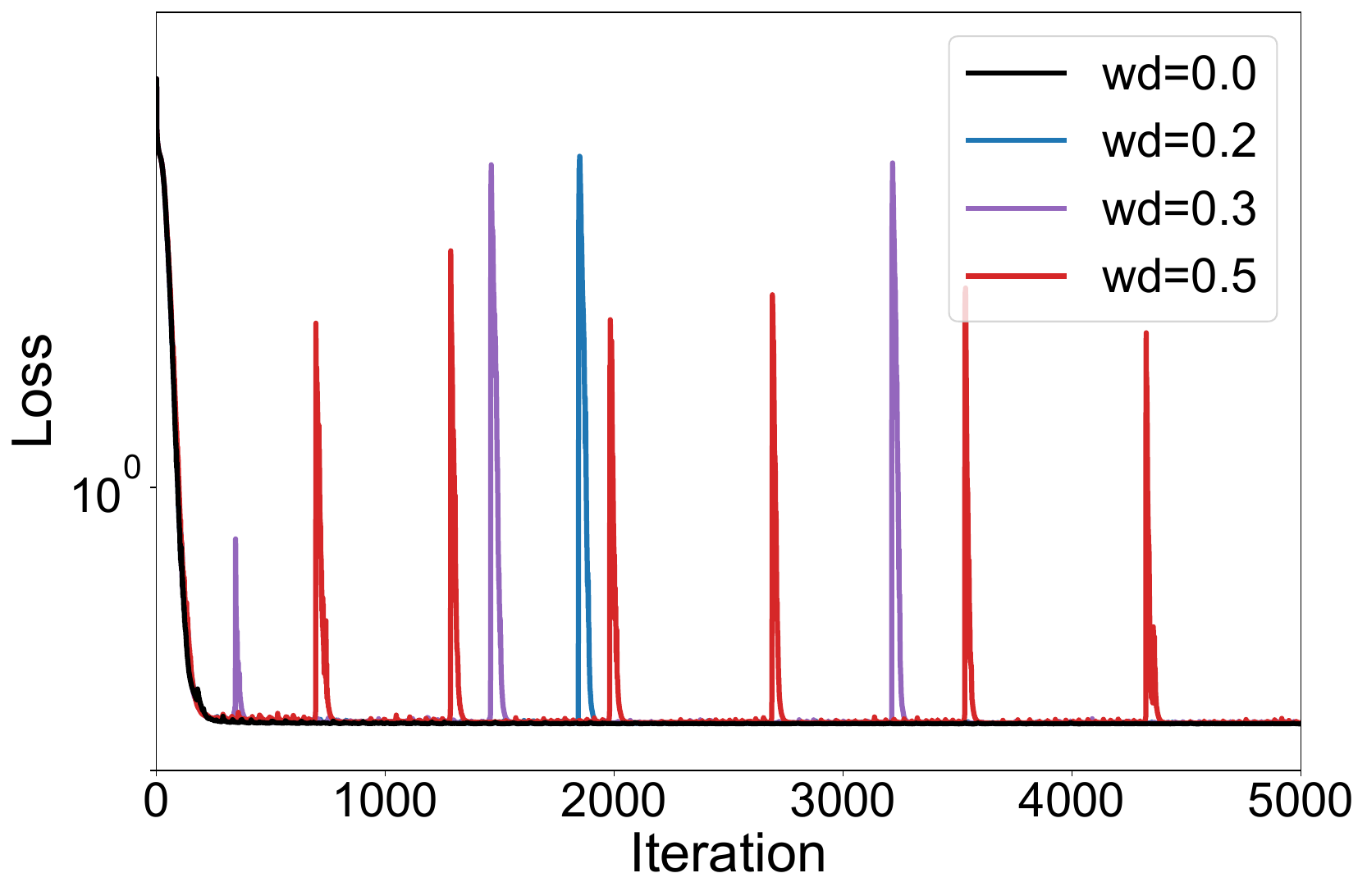}
        \caption{\textbf{Training loss on the synthetic $3x\!\to\!x$ next-token prediction task under varying weight decay.}
    Shown are full-batch training loss trajectories for different weight decay coefficients.
    Each iteration corresponds to a single full-batch parameter update.
    All other model and optimization settings are kept fixed.
    }

    \label{fig:transformer_3x_to_x_ntp_trainloss}
\end{figure}

We consider a fully controlled regression task designed to reproduce loss spikes in a minimal and analytically tractable setting.
The model is a three-layer FNN that maps a two-dimensional input to a one-dimensional output, with two ReLU-activated hidden layers followed by a linear readout.

Normalization (either BN or LN) is applied after each of the first two affine transformations, while the final output layer remains unnormalized.
The network is defined as
\[
\begin{aligned}
h_1 &= \mathrm{ReLU}\!\left(\texttt{Norm}(W_1 x + b_1)\right), \\
h_2 &= \mathrm{ReLU}\!\left(\texttt{Norm}(W_2 h_1 + b_2)\right), \\
\hat{y} &= W_3 h_2 + b_3,
\end{aligned}
\]
where $W_1 \in \mathbb{R}^{D \times 2}$, $b_1 \in \mathbb{R}^{D}$,
$W_2 \in \mathbb{R}^{D \times D}$, $b_2 \in \mathbb{R}^{D}$,
$W_3 \in \mathbb{R}^{1 \times D}$, and $b_3 \in \mathbb{R}$.

The regression target is defined as
\[
y = x_1 + 2x_2,
\]
with input $(x_1, x_2) \in \mathbb{R}^2$.

We adopt a two-dimensional input to avoid degeneracies that can arise in scale-invariant networks with one-dimensional inputs.
In the one-dimensional case, the first-layer weight reduces to a scalar, which admits no nontrivial orthogonal direction; under normalization, this can induce vanishing gradients and effectively freeze the corresponding parameters, obscuring the dynamics of interest.

The model is optimized using full-batch gradient descent with mean squared error (MSE) loss.
For learning rates $\eta \in \{0.01,\,0.03\}$, we sweep the weight decay coefficient and record the training loss as a function of epoch.
The resulting loss trajectories are shown in Fig.~\ref{fig:loss_2x1y}.

\subsection{Controlled synthetic next-token prediction with Transformer}
\label{app:synthetic_transformer_details}

\paragraph{Synthetic "3x to x" next-token prediction task.}
We study a controlled synthetic next-token prediction task following the one-anchor identity-learning formulation introduced in \cite{zhang2024anchorfunctiontypebenchmark}. 
Each input sequence has fixed length $n=9$ and contains a designated anchor token ``3'' that appears exactly once among the first $n-1$ positions. 
The learning objective is to predict the token immediately following the anchor, i.e., for an input sequence $(\ldots,3,x,\ldots)$ the correct output is $x$, while all other tokens in the sequence are irrelevant.
Tokens are represented using one-hot encodings over a finite vocabulary, and supervision is applied only at the final sequence position via a cross-entropy loss, turning the problem into a well-defined next-token classification task.

\paragraph{Model and optimization.}
We train a $4$-layer, single-head Transformer on this task.
The model dimension is set to $d_{\text{model}}=400$, with feedforward width $d_{\text{ff}}=1200$, and attention dimensions $d_Q=d_K=d_V=64$.
Training is performed using AdamW with a fixed learning rate $\eta=10^{-3}$ and momentum parameters $(\beta_1,\beta_2)=(0.9,0.999)$.
We apply global-norm gradient clipping with threshold $1$.
The batch size is set equal to the full training set size ($1800$), so that each epoch corresponds to a full-batch parameter update.
Following the main text, we sweep the weight decay coefficient while keeping all other hyperparameters fixed. Figure~\ref{fig:transformer_3x_to_x_ntp_trainloss} reports the training loss trajectories of the $3x\!\to\!x$ task under different weight decay coefficients. While all configurations optimize the same objective, varying weight decay leads to markedly different optimization dynamics, highlighting the sensitivity of this controlled task to regularization strength.

\section{Additional Empirical Results}
\label{app:additional_empirical_results}

\subsection{Weight-Norm Shrinkage in Scale-Invariant Layers}

We plot the evolution of weight norms across training epochs for different layers, where \texttt{fc1} and \texttt{fc2} are \emph{scale-invariant} layers and \texttt{fc-out} is a \emph{non-scale-invariant} layer. We fix the learning rate at $0.03$ and consider different weight-decay settings: $0.01$, $0.03$, $0.05$, $0.1$. As shown in Fig.~\ref{fig:weightnorm_lr003}, weight decay induces a pronounced contraction of the weight norms in the scale-invariant layers: the norms of \texttt{fc1} and \texttt{fc2} decrease substantially over training, and the magnitude of this decrease is markedly larger than that observed for the \texttt{fc-out} layer. This contrast indicates that, under identical optimization settings, weight decay exerts a considerably stronger norm-shrinking effect on scale-invariant layers than on the non-scale-invariant layer.


\begin{figure}[h]
\centering
    \includegraphics[width=0.5\columnwidth]{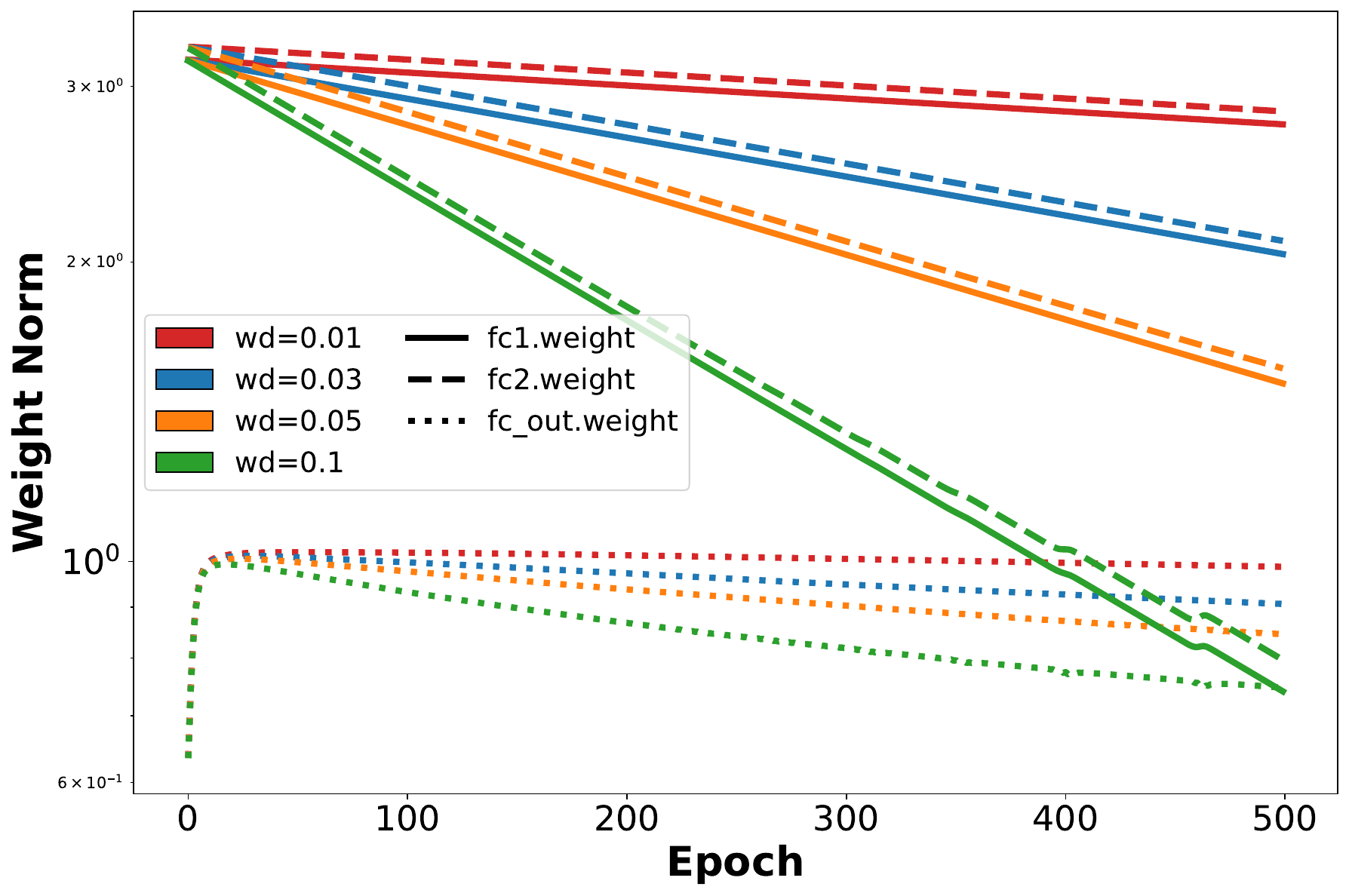}
    \caption{\textbf{Weight decay selectively contracts scale-invariant layers.}
    Layer-wise weight-norm trajectories during training at learning rate $\eta=0.03$ for different weight decay coefficients.
    The norms of the scale-invariant layers (\texttt{fc1}, \texttt{fc2}) decay significantly faster than that of the non-scale-invariant layer (\texttt{fc-out}).
    }
    \label{fig:weightnorm_lr003}
\end{figure}

\subsection{Ablation Study: Weight Decay Applied Only to Non-Scale-Invariant Layers}

Additionally, we conducted an ablation study using a three-layer FNN incorporating normalization layers. 
In this experiment, we applied weight decay exclusively to the non-scale-invariant layers. 
The resulting loss curves are presented in Figure~\ref{fig:4.5}.

Following the same protocol, we performed multiple training runs on this model using varying weight decay settings. 
We applied PCA to visualize the parameter trajectories in a low-dimensional space alongside the loss landscape. 
As illustrated in the figure, the maximum eigenvalue ($\lambda_{\max}$) measured at the end of training does not exhibit a clear increasing trend as the weight decay strength increases.
This suggests that, when weight decay is restricted to non-scale-invariant layers, stronger regularization does not lead to a systematic increase in sharpness.

When weight decay is not applied to scale-invariant layers, model stability improves significantly, and loss spikes are effectively suppressed. 
To illustrate this, we compared the loss curves and $\lambda_{\max}$ trajectories under settings with a large learning rate and large weight decay.



\begin{figure}[h]
  \vskip 0.2in
  \centering
  \begin{subfigure}[h]{0.49\columnwidth}
    \centering
    \includegraphics[width=0.8\columnwidth]{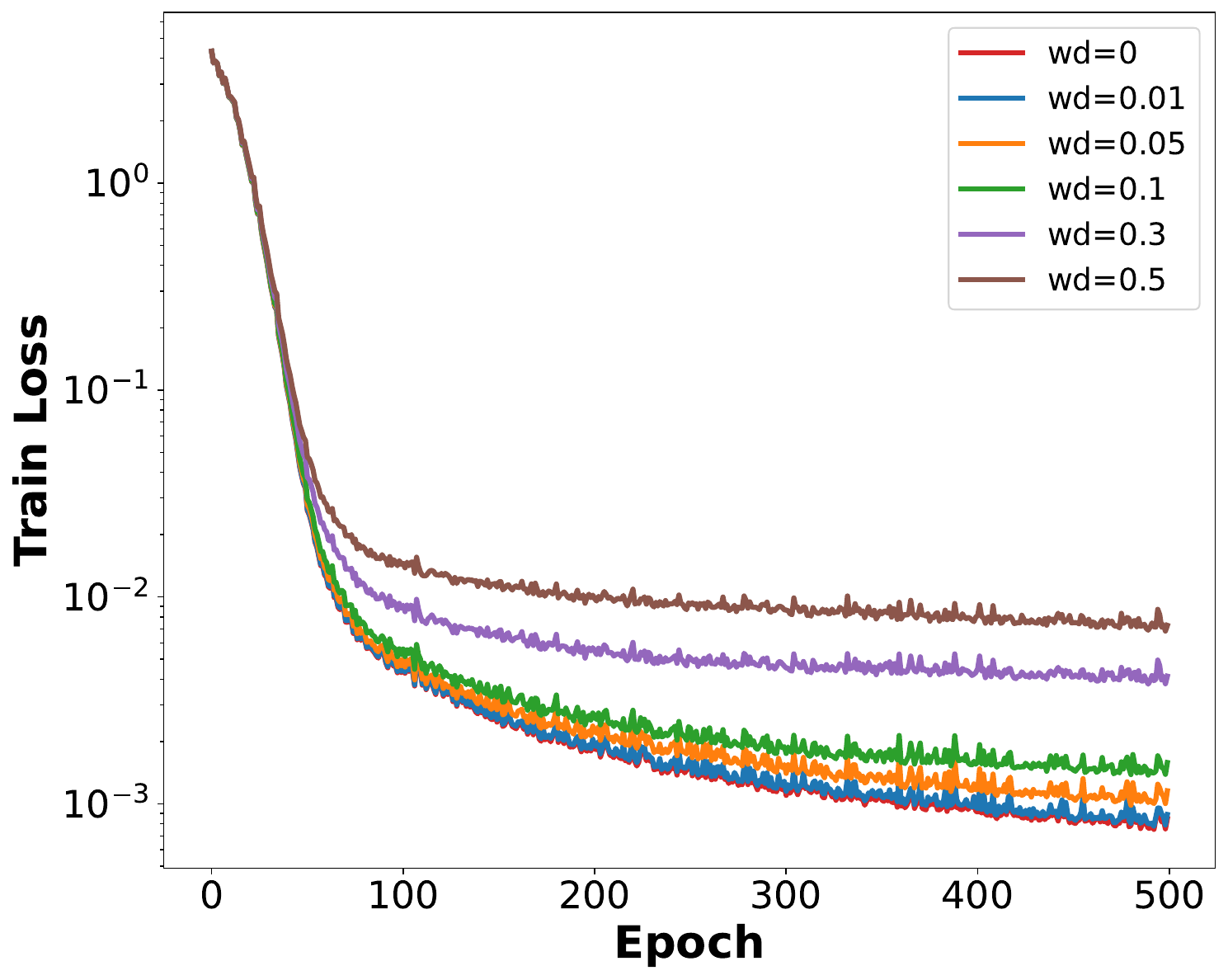}
  \end{subfigure}
  \hfill
  \begin{subfigure}[h]{0.49\columnwidth}
    \centering
    \includegraphics[width=0.8\columnwidth]{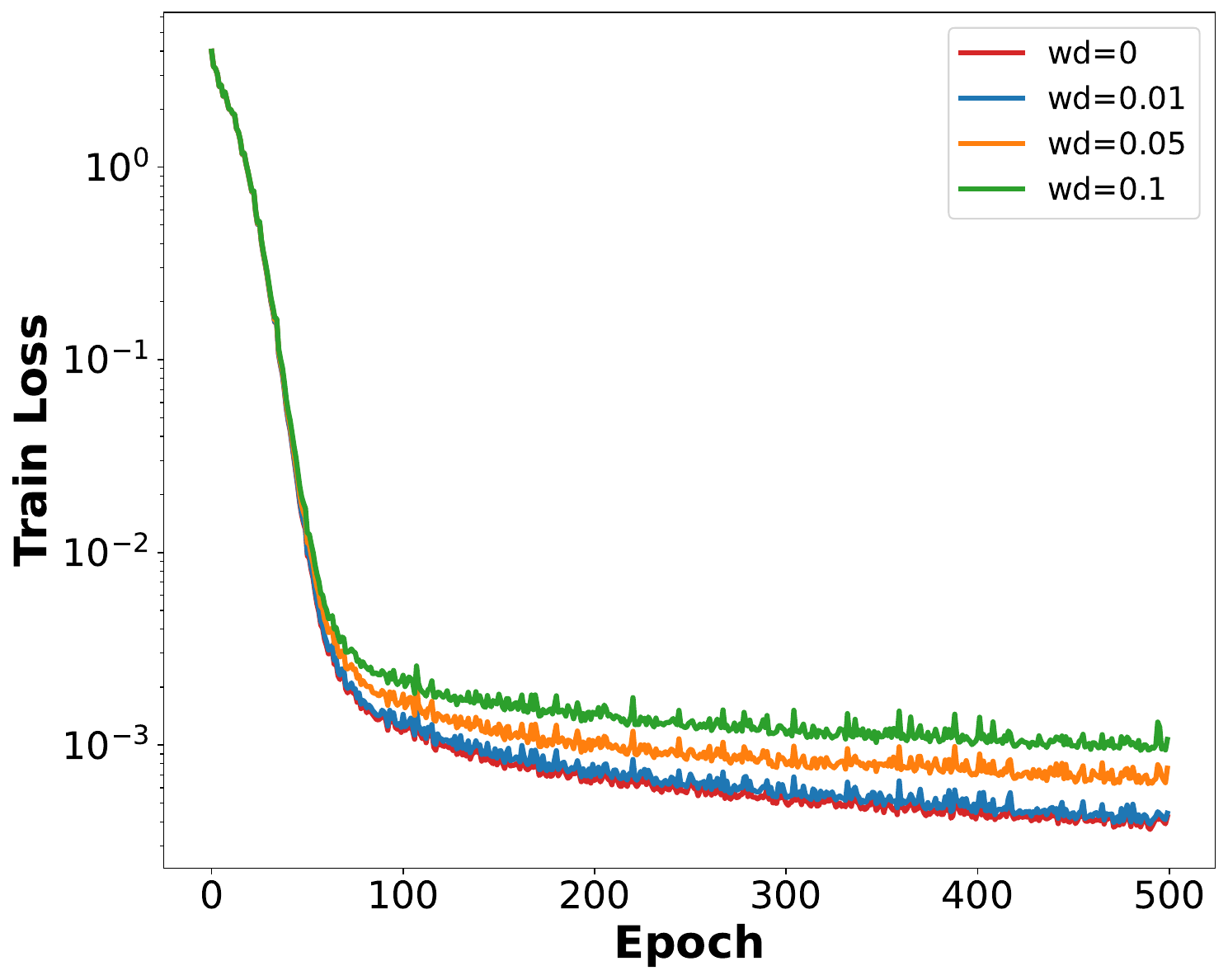}
  \end{subfigure}
  \caption{\textbf{Training loss with weight decay applied only to non-scale-invariant layers.}
    Training loss trajectories under different weight decay coefficients when regularization is applied exclusively to the non-scale-invariant layer (\texttt{fc-out}).
    Results are shown for two learning rates, $\eta=0.01$ (left) and $\eta=0.03$ (right).
    In contrast to the standard setting where weight decay is applied to all layers, the loss curves exhibit improved stability and demonstrate no pronounced loss spikes, even under relatively large learning rates and weight decay.
    }

  \label{fig:4.5}
\end{figure}



\begin{figure}[h]
  \vskip 0.2in
  \centering
  \begin{subfigure}[h]{0.49\columnwidth}
    \centering
    \includegraphics[width=0.8\columnwidth]{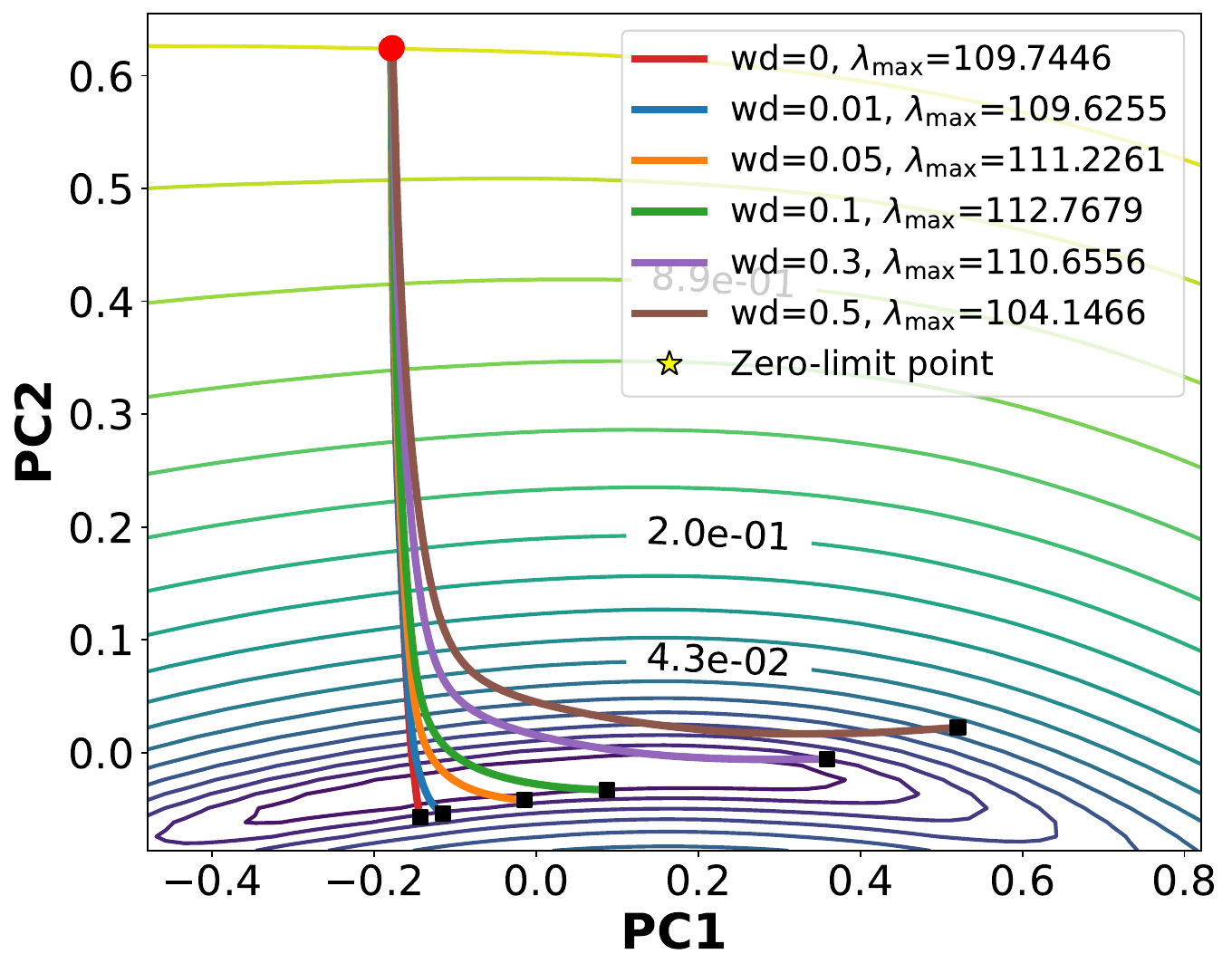}
    \caption{lr=0.01}
    \label{fig:4.5.1}
  \end{subfigure}
  \hfill
  \begin{subfigure}[h]{0.49\columnwidth}
    \centering
    \includegraphics[width=0.8\columnwidth]{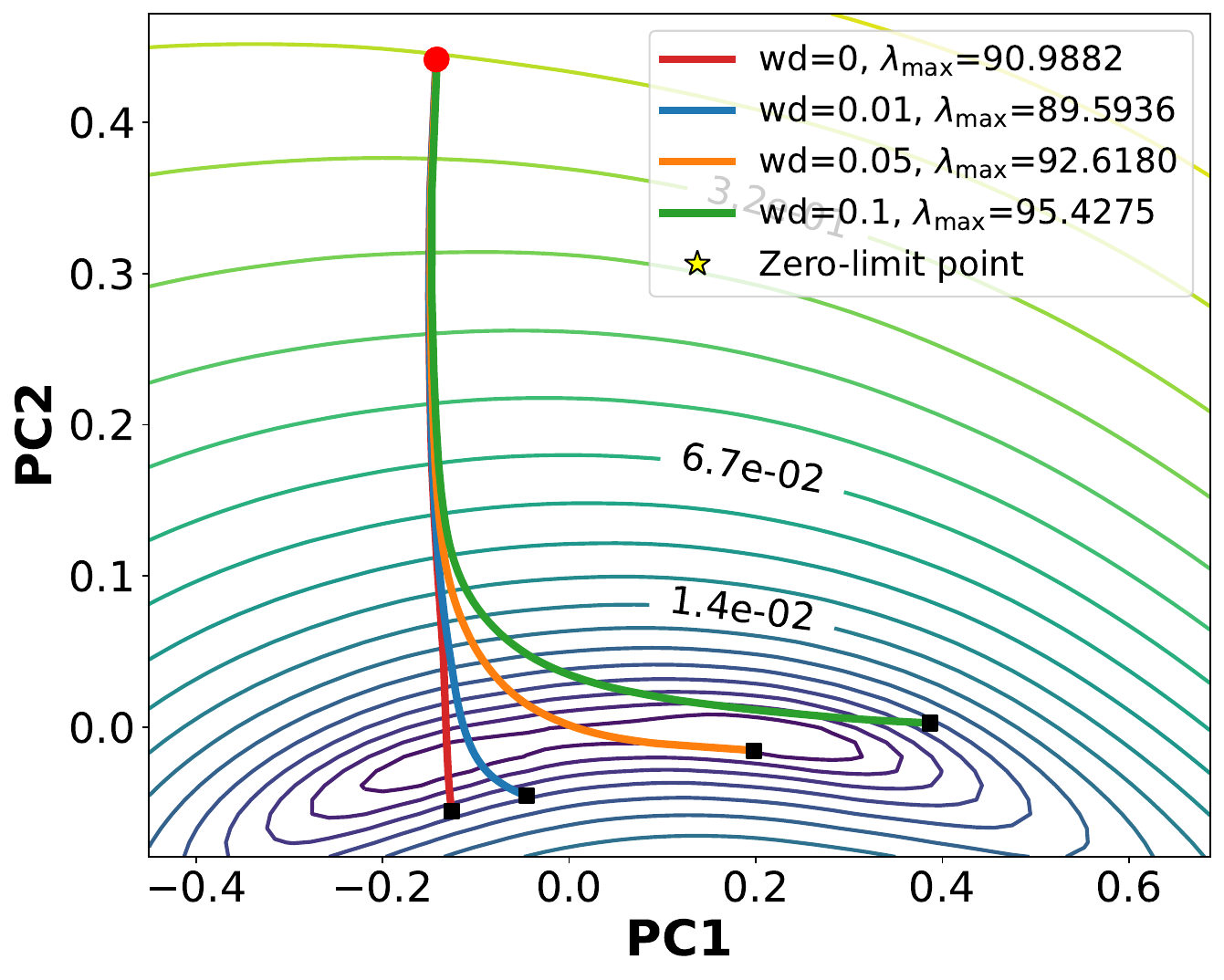}
    \caption{lr=0.03}
    \label{fig:4.5.2}
  \end{subfigure}

    \caption{\textbf{PCA visualization of loss landscapes and training trajectories with non-scale-invariant-only weight decay.}
      Two-dimensional PCA projections of parameter trajectories and the corresponding loss landscapes for the same settings as in Fig.~\ref{fig:4.5}.
      Only the non-scale-invariant layer is regularized by weight decay.
      The explained variance ratios of the first two principal components are $81.40\%$ and $13.36\%$ for $\eta=0.01$, and $85.51\%$ and $8.89\%$ for $\eta=0.03$.
      Across different weight decay values, the maximum Hessian eigenvalue $\lambda_{\max}$ measured at the end of training does not exhibit a clear increasing trend.
      }
\end{figure}

\begin{figure}[h]
  \centering
  \includegraphics[width=0.4\columnwidth]{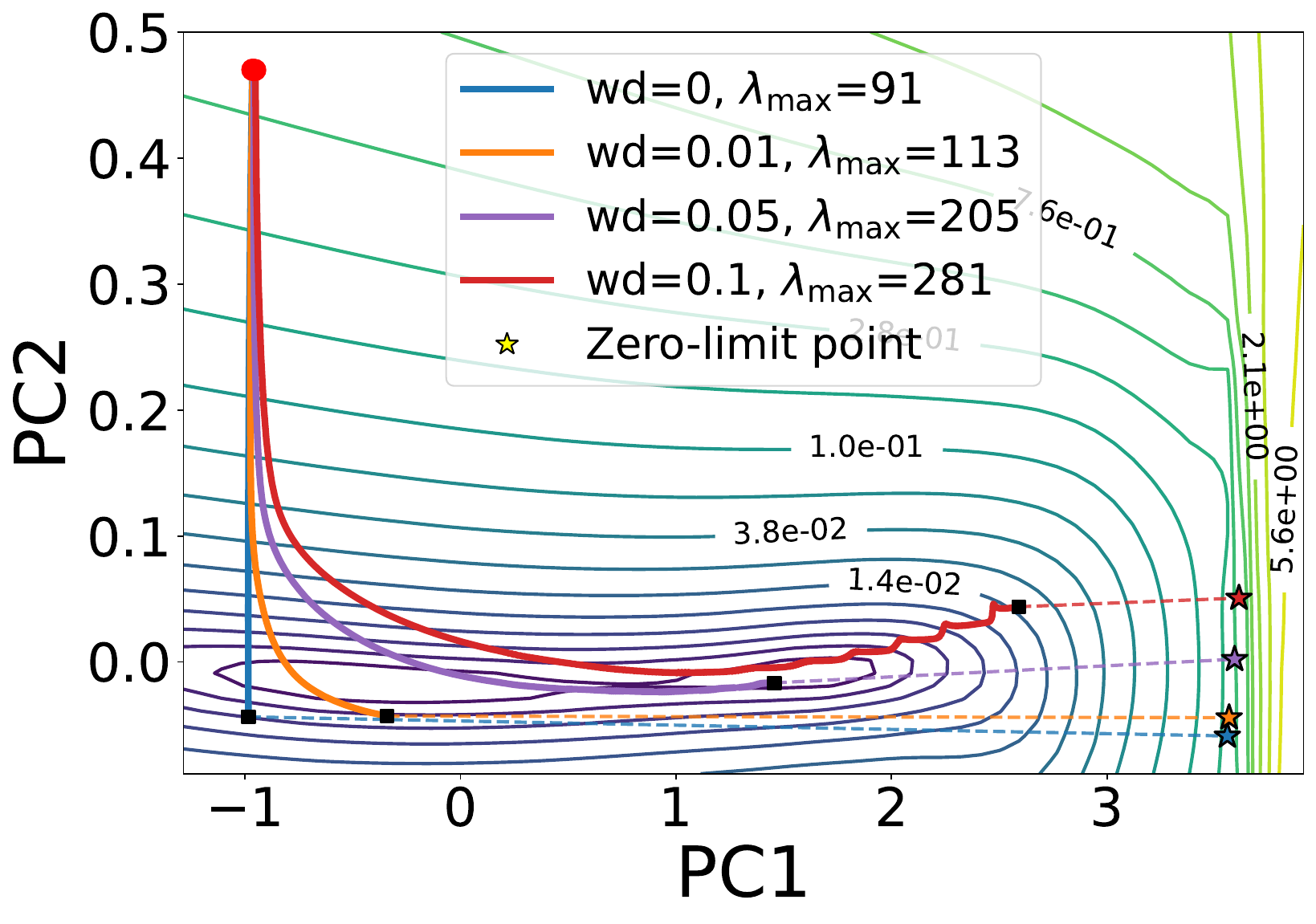}
  \caption{PCA visualization of training trajectories under varying weight decay. The black squares mark the final parameter states and $\lambda_{\max}$ denotes the largest eigenvalue of the Hessian matrix at the training endpoint. The colored stars represent the derived parameters obtained by freezing the output layer of the trained models and pushing the scale-invariant parameters to zero. Contour lines represent training loss. The explained variance ratios of PC1 and PC2 are $0.9958$ and $0.0020$, respectively.}
  \label{fig:app:traj_pca}
\end{figure}

\subsection{Supplementary Figures}

\subsection{Details of the Stability Boundary}
\label{app:boundary_details}

\begin{figure}[h]
  \vskip 0.2in
  \centering
  \begin{subfigure}[h]{0.32\columnwidth}
    \centering
    \includegraphics[width=\columnwidth]{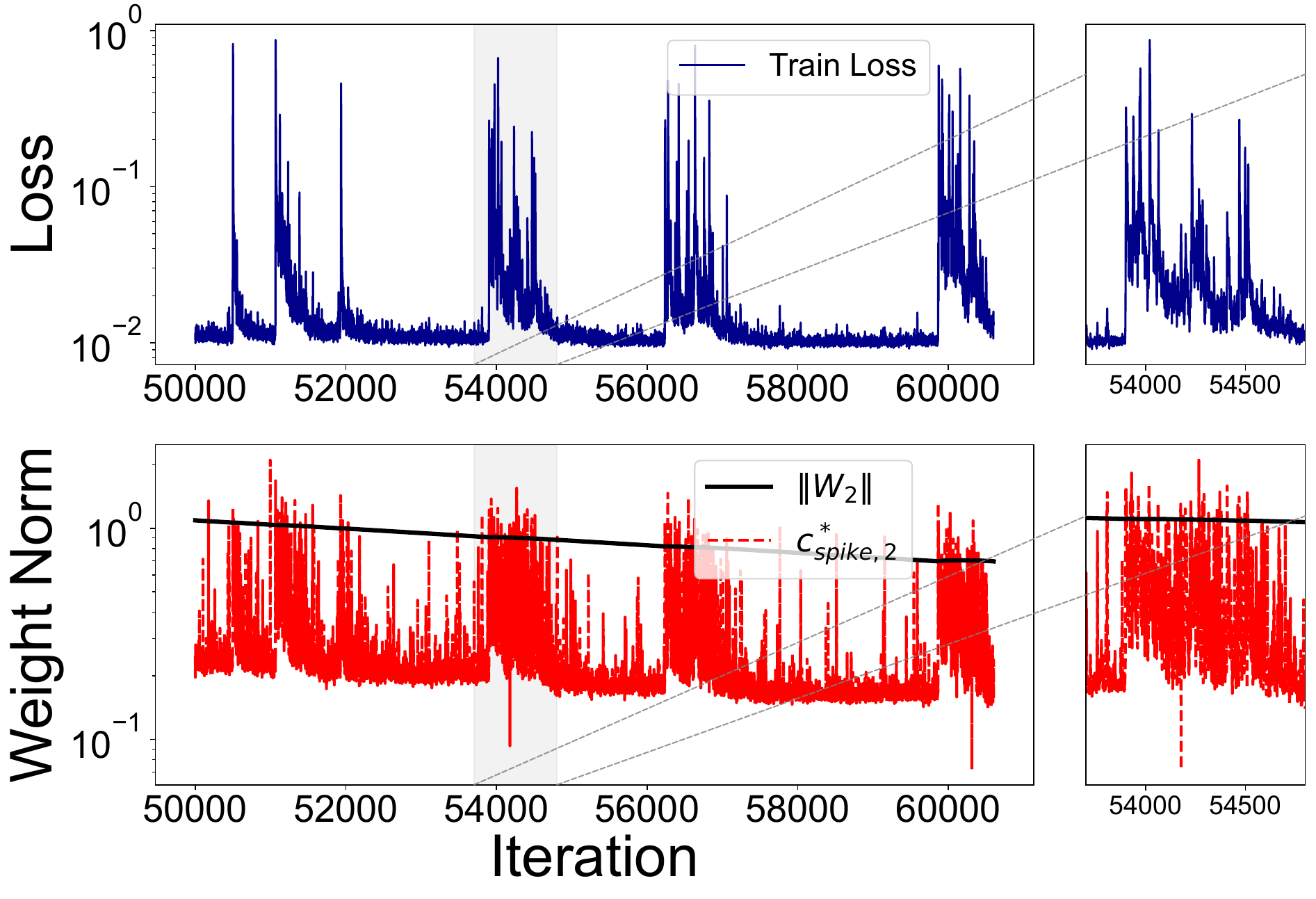}
    \caption{Second layer}
    \label{layer2_iter}
  \end{subfigure}
  \hfill
  \begin{subfigure}[h]{0.32\columnwidth}
    \centering
    \includegraphics[width=\columnwidth]{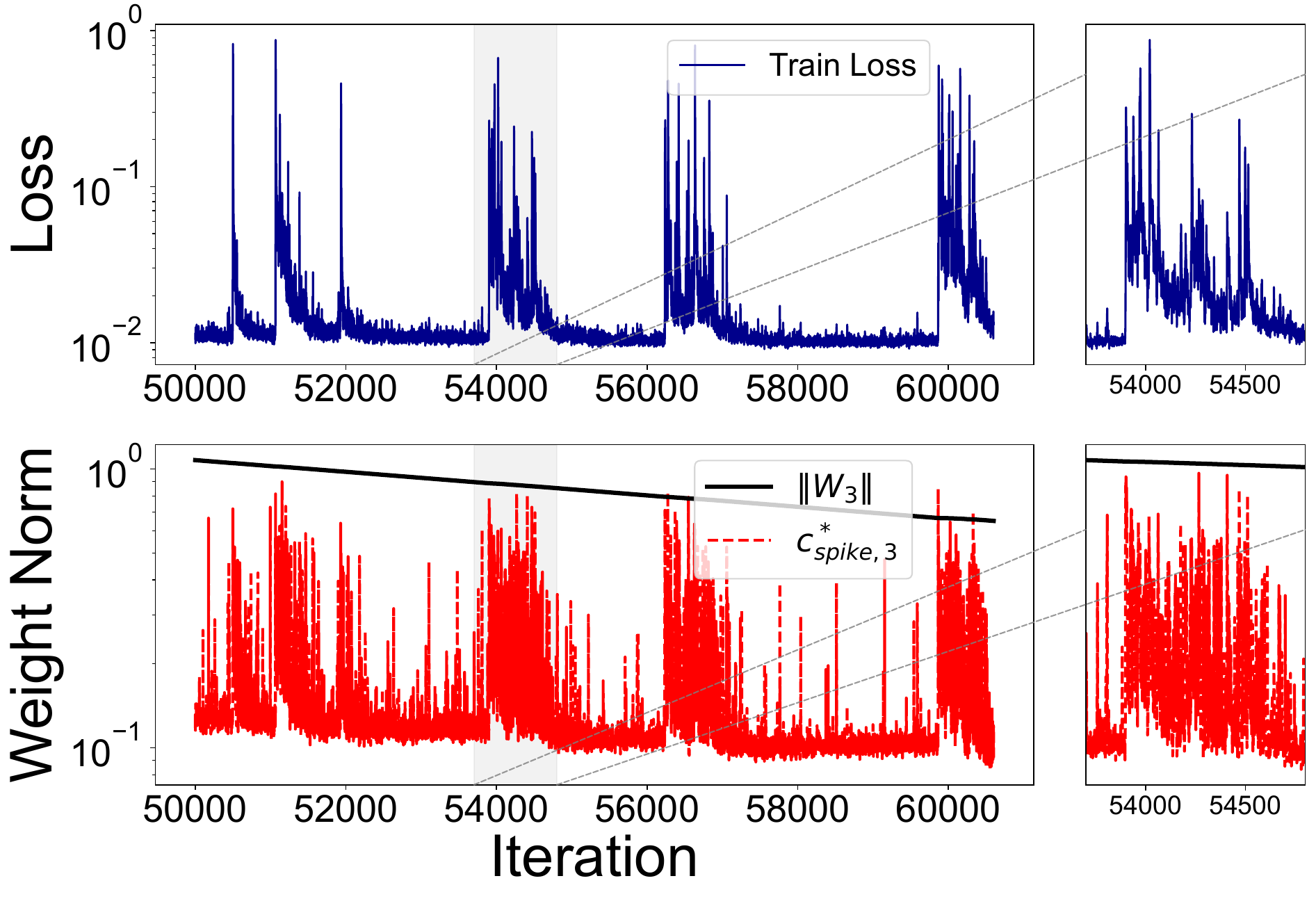}
    \caption{Third layer}
    \label{layer3_iter}
  \end{subfigure}
  \hfill
  \begin{subfigure}[h]{0.32\columnwidth}
    \centering
    \includegraphics[width=\columnwidth]{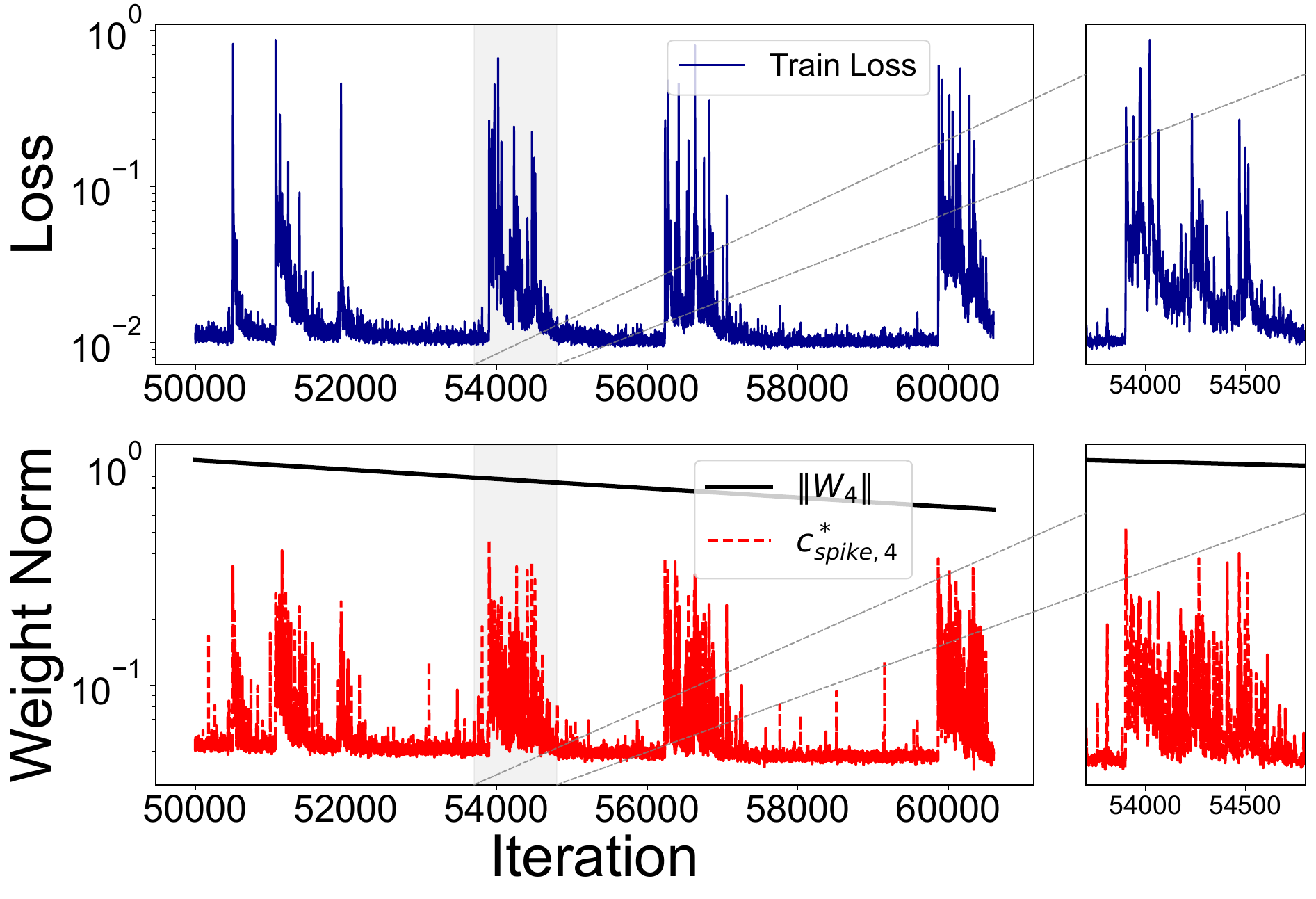}
    \caption{Fourth layer}
    \label{layer4_iter}
  \end{subfigure}
  \caption{\textbf{Layer-wise stability boundary analysis.}
  Late-training evolution of the weight norms and predicted stability boundaries for the second, third, and fourth scale-invariant layers in the MNIST experiment ($\eta=0.005$, $\lambda=0.01$).
  The same interval-filtering procedure as in the main text is applied.
  Compared to the first layer, boundary crossings in deeper layers are less frequent and less persistent, though sustained excursions remain temporally correlated with instability in the training dynamics.
  }
  \label{fig:boundary_other_layers}
\end{figure}

In this appendix, we provide additional layer-wise results for the stability boundary analysis discussed in Main text Section~4-Theoretical Analysis. As explained in the main text, the remaining scale-invariant layers display qualitatively similar but less frequent behavior.

Figure~\ref{fig:boundary_other_layers} reports the late-training dynamics of the weight norms and corresponding stability boundaries for the second, third, and fourth scale-invariant layers in the MNIST experiment, under learning rate $\eta=0.005$ and weight decay $\lambda=0.01$.
For consistency with the main text, the same filtering protocol is applied to the detected unstable intervals: adjacent excursions separated by fewer than $30$ iterations are merged, and intervals shorter than $200$ iterations are discarded.
This filtering suppresses short-lived boundary crossings that do not correspond to macroscopic instabilities.

Across these layers, boundary crossings are less frequent and typically shorter-lived than those observed in the first layer.
Nevertheless, when sustained excursions into the predicted unstable regime occur—i.e., when the weight norm falls below the layer-specific stability boundary—they remain temporally aligned with observable disturbances in the optimization trajectory.
These results support the layer-wise nature of the proposed stability criterion, while highlighting that the first layer plays a dominant role in triggering instability in the settings considered.

\subsection{Details of Top-Eigenvector Decomposition}
\label{app:transformer_modulewise_eig}

Let $\theta$ denote the vector of all learnable parameters, and let $H := \nabla^2_{\theta}\mathcal{L}$ be the Hessian of the training loss.
At selected training iterations, we compute the leading Hessian eigenvector
{\small\begin{equation}
v := \arg\max_{\|u\|_2=1} u^\top H u,
\end{equation}}
corresponding to the maximum eigenvalue of $H$.
We decompose this eigenvector into contributions from disjoint parameter modules.

For each module $m$ (e.g., embeddings; attention projections $W_Q, W_K, W_V$ and the output projection; MLP blocks; etc.), let $\mathcal{I}_m$ denote the index set of parameters belonging to that module.
We define the projected component $v_m$ by restricting $v$ to the coordinates in $\mathcal{I}_m$ (equivalently, $v_m = P_m v$, where $P_m$ denotes the corresponding coordinate projection operator).
We then report the module-wise fraction
{\small\begin{equation}
\phi_m := \frac{\|v_m\|_2}{\|v\|_2},
\end{equation}}
which quantifies the relative contribution of module $m$ to the top eigenvector, as a function of training iteration.

In Figs.~8(Left) and Figs.~8(Right), different colors correspond to different parameter modules. For the ablation shown in Fig.~8(Right), weight decay is set to zero for all parameters in the MLP blocks, while the same weight decay coefficient as in the baseline is applied to all remaining modules. All other hyperparameters are kept identical to the baseline configuration (see Appendix~\ref{app:synthetic_transformer_details} for details).

\section{EoS Stability and Loss Spike Condition}
\label{app:spike_derivation}

\subsection{Linearization and the EoS stability criterion}
Consider the GD update $w^{+} = w - \eta \nabla \mathcal{L}(u,v)$ and linearize it around the current iterate.
For a small perturbation $\delta w$, we obtain
{\small\begin{equation}
\delta w^{+} \approx (I - \eta H(u,v))\,\delta w,
\end{equation}}
where $H(u,v)$ is the Hessian of $\mathcal{L}(\cdot,v)$ at $w$.
Linear stability requires the spectral radius of $(I-\eta H(u,v))$ to be at most $1$.
When $H(u,v)$ is (locally) positive semidefinite, this reduces to requiring
$|1-\eta\lambda_i|\le 1$ for every eigenvalue $\lambda_i$ of $H(u,v)$, i.e.,
$0\le \eta\lambda_i \le 2$, hence
{\small\begin{equation}
\eta\,\lambda_{\max}(H(u,v)) \le 2,
\end{equation}}
which is exactly Eq.5 in the main text.

\subsection{Second-order loss change and the spike condition}
Linear instability is not equivalent to an immediate increase of the objective.
To characterize a one-step loss increase, we examine the loss change under the actual GD direction.
Let $g := \nabla \mathcal{L}(u,v)$ and $w^{+}=w-\eta g$.
A second-order Taylor expansion yields
{\small\begin{equation}
\label{eq:taylor_spike}
\mathcal{L}(w^{+},v)
=
\mathcal{L}(u,v)
-\eta\|g\|^{2}
+\frac{\eta^{2}}{2}\, g^{\top}H(u,v)\,g
+O(\eta^{3}).
\end{equation}}
Define the gradient-direction curvature (Rayleigh quotient)
{\small\begin{equation}
\label{eq:lambda_grad_w}
\lambda_{\mathrm{grad}}(H(u,v))
:=
\begin{cases}
\dfrac{g^{\top}H(u,v)\,g}{\|g\|^{2}}, & \text{if } \|g\|>0,\\[6pt]
0, & \text{if } \|g\|=0.
\end{cases}
\end{equation}}
Substituting \eqref{eq:lambda_grad_w} into \eqref{eq:taylor_spike} gives, up to second order,
{\small\begin{equation}
\Delta \mathcal{L}
:=
\mathcal{L}(w^{+},v)-\mathcal{L}(u,v)
\approx
-\eta\|g\|^{2}\Bigl(1-\frac{\eta}{2}\lambda_{\mathrm{grad}}(H(u,v))\Bigr).
\end{equation}}
Therefore, the second-order approximation predicts a \emph{loss spike} ($\Delta\mathcal{L}>0$) whenever
{\small\begin{equation}
\label{eq:spike_condition}
\eta\,\lambda_{\mathrm{grad}}(H(u,v)) > 2.
\end{equation}}
Compared with the EoS condition based on $\lambda_{\max}(H)$, the spike criterion \eqref{eq:spike_condition} is directional:
it requires the gradient to align sufficiently with sharp curvature directions so that the quadratic term in \eqref{eq:taylor_spike} dominates the linear decrease.

\subsection{Implication under scale-invariant homogeneity}
For scale-invariant parameters, the homogeneity/scaling property derived in the main text implies that the relevant curvature scales as $1/\|u\|^{2}$.
In particular, under the scaling $u\mapsto \alpha u$, the Hessian block satisfies
$H_{uu}(\alpha u,v)=\alpha^{-2}H_{uu}(u,v)$.
The corresponding Rayleigh quotient inherits the same scaling:
$\lambda_{\mathrm{grad}}(H_{uu}(\alpha u,v))
=\alpha^{-2}\lambda_{\mathrm{grad}}(H_{uu}(u,v))$.
This allows translating \eqref{eq:spike_condition} into a critical norm threshold, motivating the spike boundary in Definition~\ref{def:spike_boundary}.

\section{Multi-Seed Validation and Quantitative Robustness Analysis}
\label{app:multi_seed_validation}

To evaluate the robustness of our findings across random initializations, we conduct 100 independent runs with different random seeds for each experimental setting. We quantitatively evaluate the predictive accuracy of the proposed weight-norm criticality boundary $c^*$ and compute confidence intervals to assess the statistical reliability of the results.

\begin{figure*}[t]
    \centering

    \begin{subfigure}[t]{0.32\textwidth}
        \centering
        \includegraphics[width=\linewidth]{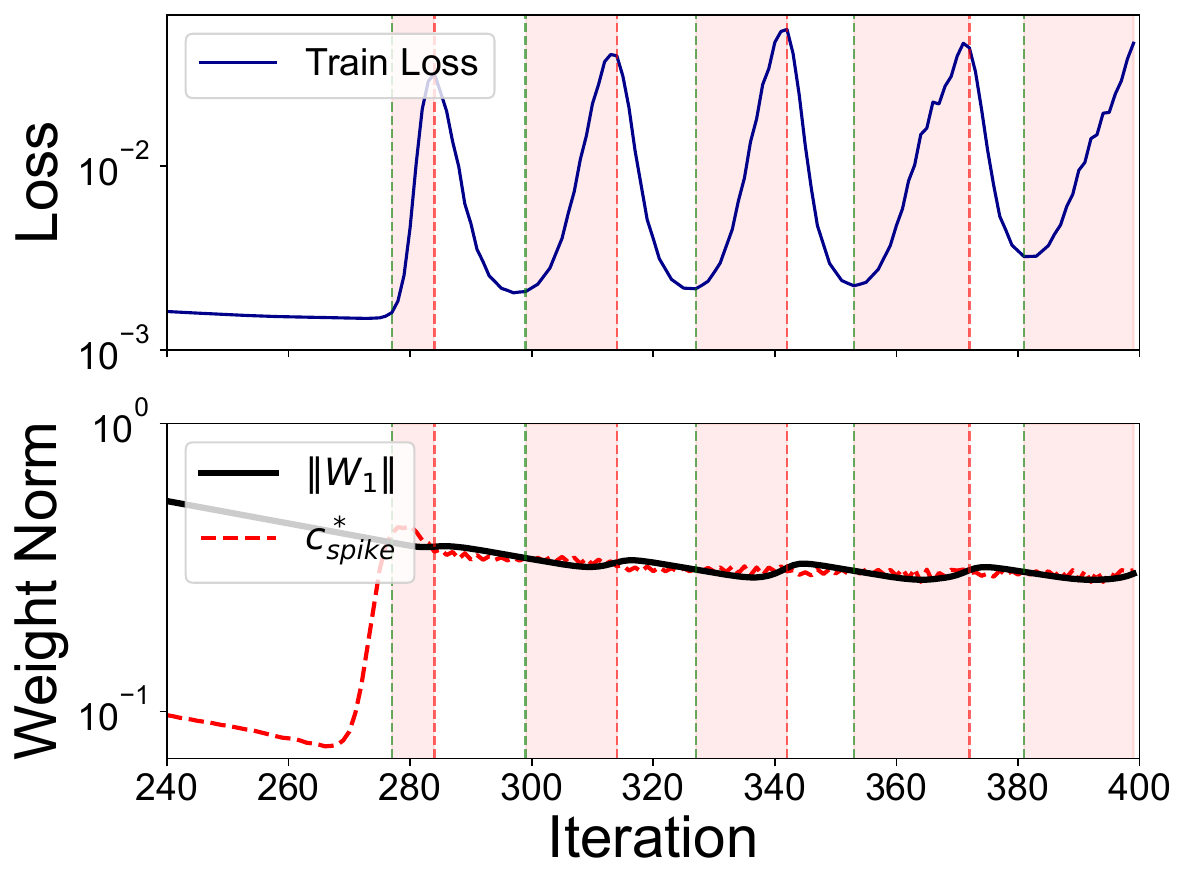}
        \caption{$\mathrm{wd}=0.3$, seed $=0$}
        \label{fig:multiseed_wd03_seed0}
    \end{subfigure}
    \hfill
    \begin{subfigure}[t]{0.32\textwidth}
        \centering
        \includegraphics[width=\linewidth]{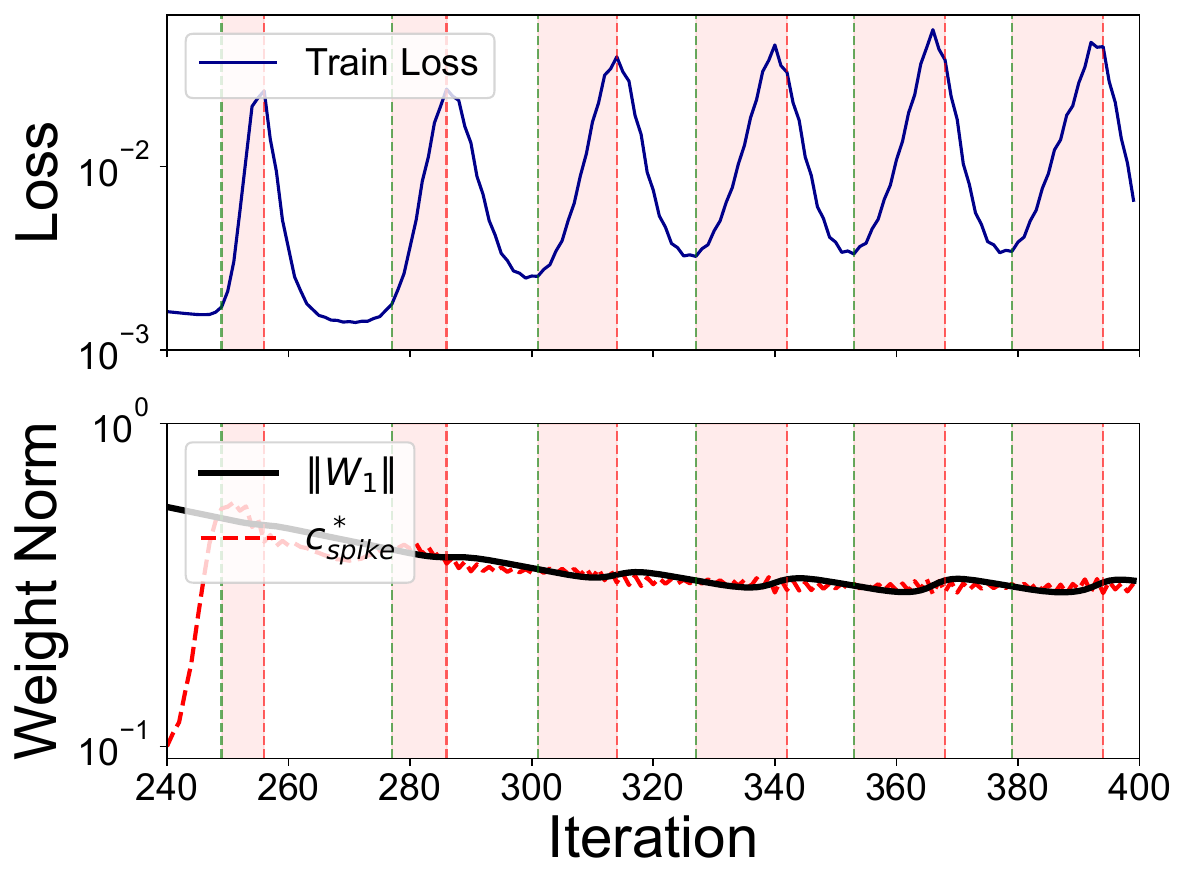}
        \caption{$\mathrm{wd}=0.3$, seed $=1$}
        \label{fig:multiseed_wd03_seed1}
    \end{subfigure}
    \hfill
    \begin{subfigure}[t]{0.32\textwidth}
        \centering
        \includegraphics[width=\linewidth]{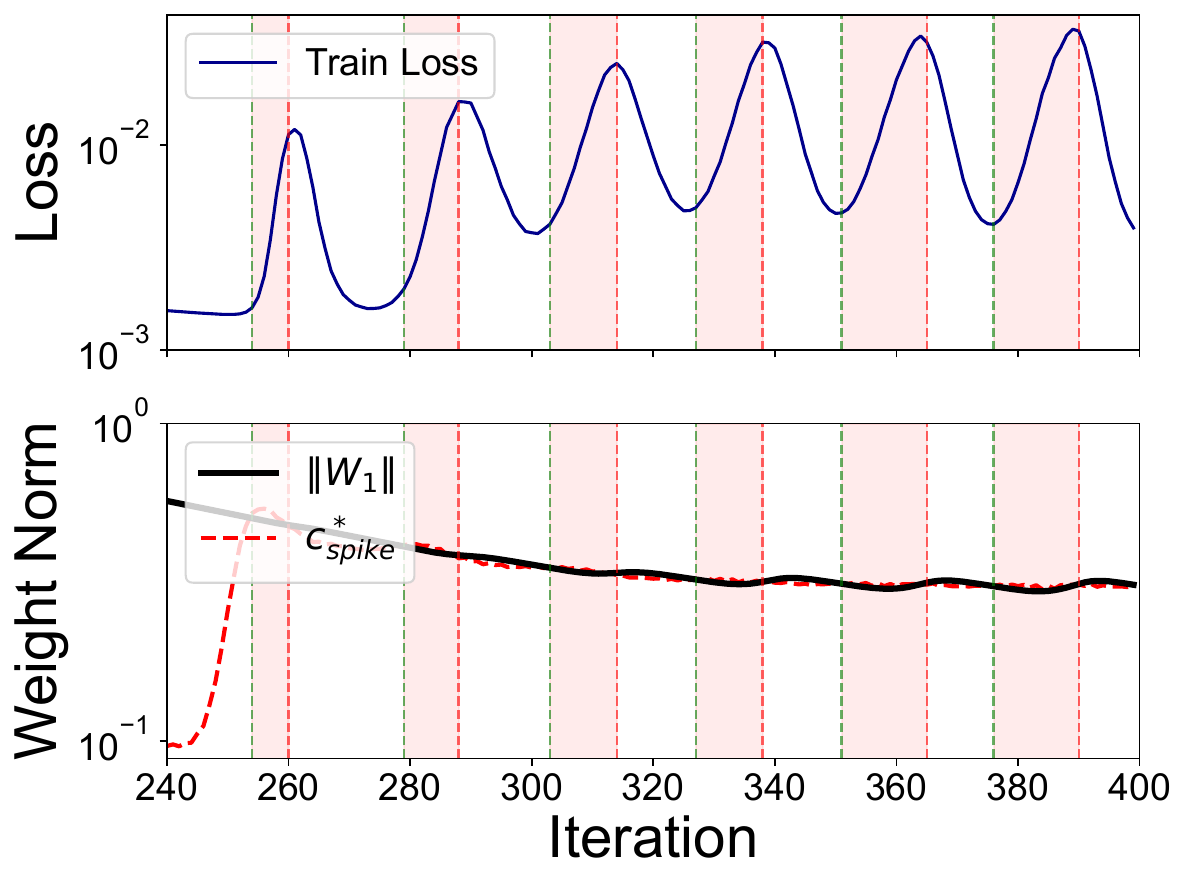}
        \caption{$\mathrm{wd}=0.3$, seed $=2$}
        \label{fig:multiseed_wd03_seed2}
    \end{subfigure}

    \vspace{0.5em}

    \begin{subfigure}[t]{0.32\textwidth}
        \centering
        \includegraphics[width=\linewidth]{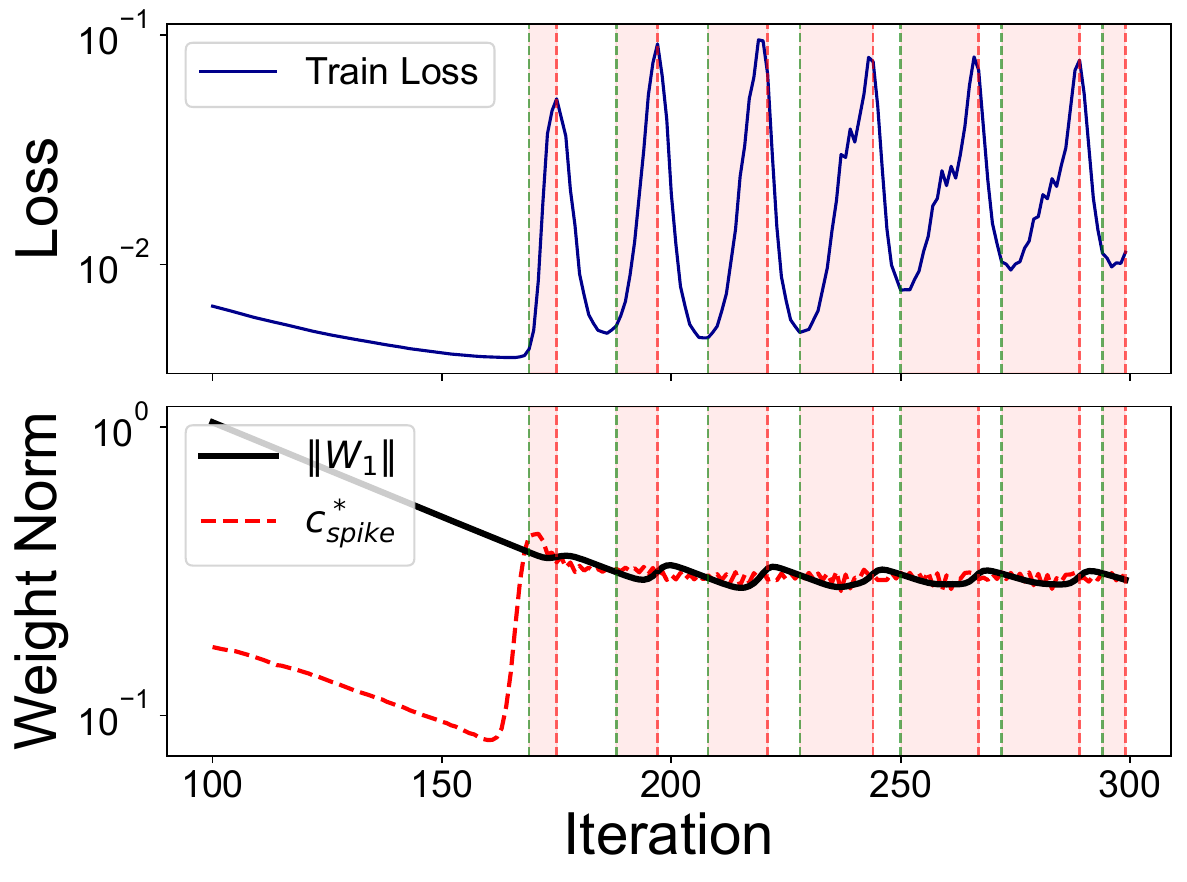}
        \caption{$\mathrm{wd}=0.5$, seed $=0$}
        \label{fig:multiseed_wd05_seed0}
    \end{subfigure}
    \hfill
    \begin{subfigure}[t]{0.32\textwidth}
        \centering
        \includegraphics[width=\linewidth]{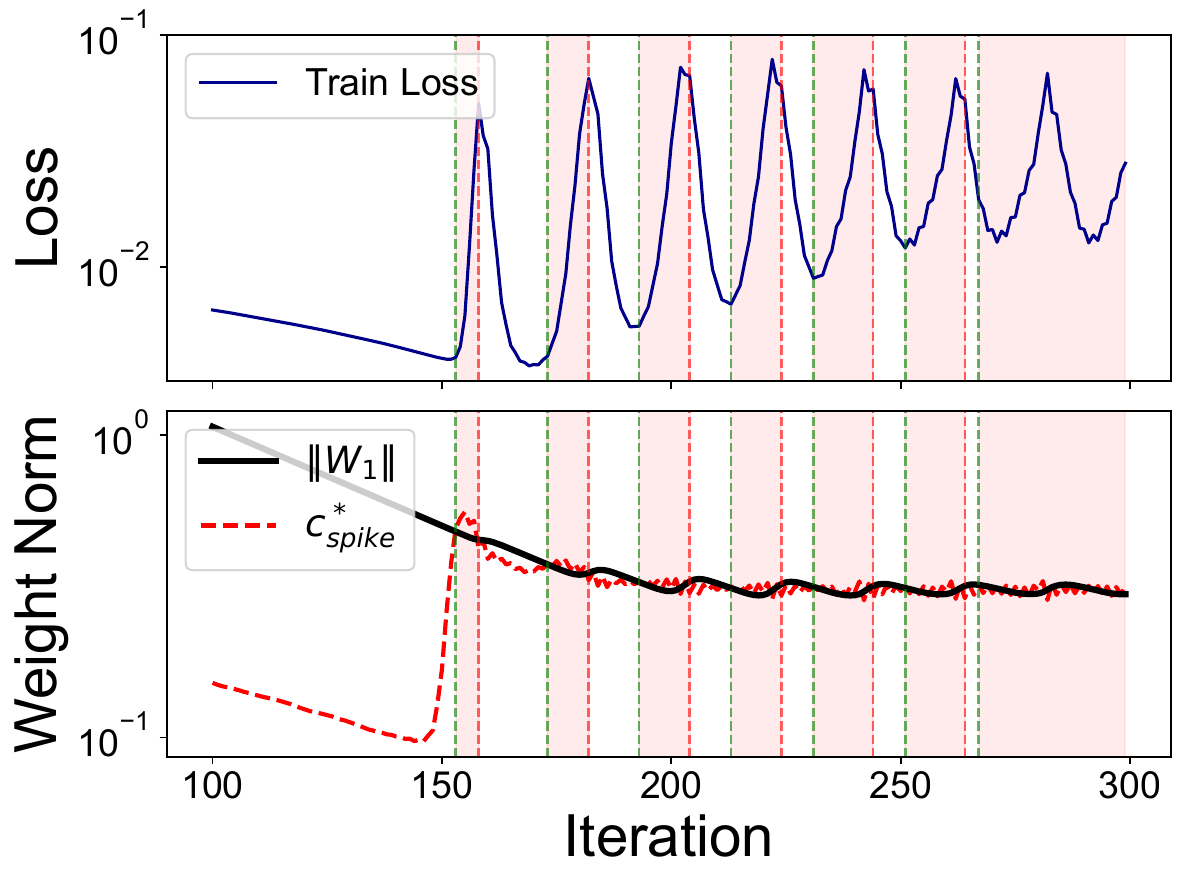}
        \caption{$\mathrm{wd}=0.5$, seed $=1$}
        \label{fig:multiseed_wd05_seed1}
    \end{subfigure}
    \hfill
    \begin{subfigure}[t]{0.32\textwidth}
        \centering
        \includegraphics[width=\linewidth]{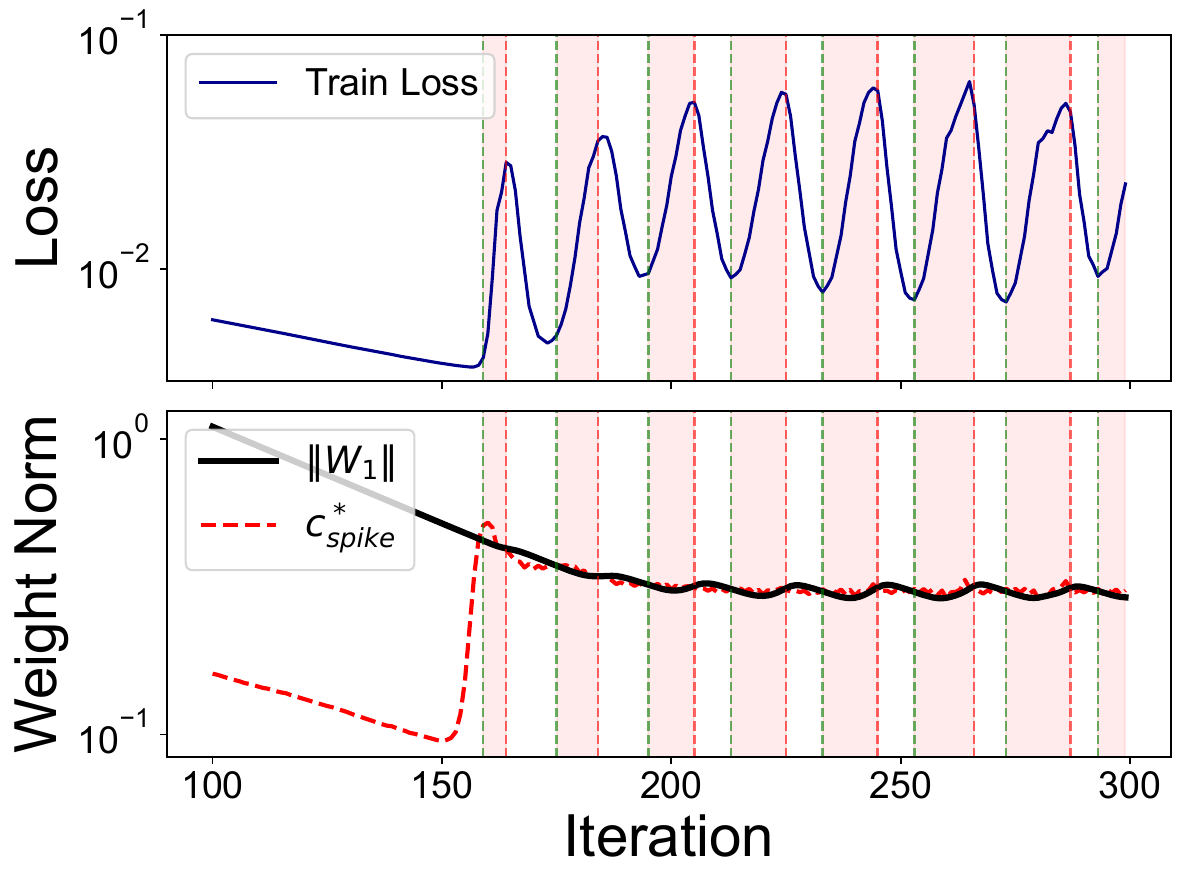}
        \caption{$\mathrm{wd}=0.5$, seed $=2$}
        \label{fig:multiseed_wd05_seed2}
    \end{subfigure}

    \caption{
    Representative training trajectories under different random seeds
    with a fixed learning rate of $\mathrm{lr}=0.03$.
    The first and second rows correspond to weight decay values
    $\mathrm{wd}=0.3$ and $\mathrm{wd}=0.5$, respectively, while the
    three columns correspond to seeds $0$, $1$, and $2$.
    Across different random initializations and weight decay values,
    loss spikes consistently occur after the weight norm crosses the
    theoretically predicted critical boundary $c^*$.
    }
    \label{fig:multi_seed_validation}
\end{figure*}

\paragraph{Multi-seed consistency.}
Across all 100 random seeds and the evaluated architectures, loss spikes consistently occur after the weight norm of a scale-invariant component falls below its corresponding theoretical critical boundary $c^*_{spike}$. As illustrated in Fig.~\ref{fig:multi_seed_validation}, this behavior remains consistent across different random initializations, demonstrating that the observed phenomenon is not caused by a particular random seed or experimental configuration, but instead reflects a robust property of the underlying training dynamics.

\paragraph{Quantitative predictive accuracy.}
Among the 615 observed training instabilities, 606 are preceded by the corresponding weight norm crossing the theoretical critical boundary. This yields a predictive accuracy of
\[
    \frac{606}{615} \times 100\% = 98.54\%.
\]
Moreover, no false-positive predictions are observed in these experiments, corresponding to a false-positive rate of $0\%$.

\paragraph{Confidence intervals.}
We additionally quantify the statistical uncertainty across random seeds. The trajectory plots are augmented with shaded regions representing the 95\% confidence intervals across runs. The predictive accuracy of weight-norm criticality is 98.54\% ($606/615$), with a 95\% confidence interval of
\[
    [0.976,\,0.995].
\]
These results further demonstrate the statistical reliability of the proposed critical boundary.

Overall, the multi-seed experiments and quantitative evaluations provide strong empirical evidence that weight-norm criticality reliably predicts training instabilities across random initializations and model architectures.

\section{Transformer Experiments with Different Normalization Methods}
\label{app:transformer_normalization}

Our theoretical analysis focuses on the strict scale invariance of a subset of model parameters. Consequently, the current theoretical framework cannot be directly applied to Transformer architectures, where residual connections and the placement of normalization layers generally break strict parameter-wise scale invariance.

Nevertheless, loss spikes are commonly observed in practical Transformer training under the combined effects of normalization and weight decay. To examine whether the effect of weight decay remains consistent across different normalization architectures, we conduct additional experiments using Pre-LayerNorm, Post-LayerNorm, Sandwich-Norm, RMSNorm, and DeepNorm. All experiments are performed on the $3x\!\rightarrow\!x$ tasks using the same experimental setup as described in Appendix~\ref{app:exp_details}, and each setting is evaluated across multiple random seeds.

Representative training trajectories are shown in Fig.~\ref{fig:transformer_norm_robustness}. Across all normalization methods, increasing weight decay consistently results in more frequent loss spikes. This trend remains stable across different random initializations, suggesting that the observed instability is not specific to a particular normalization architecture or random seed.

\begin{figure*}[t]
    \centering

    \begin{subfigure}[t]{0.32\textwidth}
        \centering
        \includegraphics[width=\linewidth]{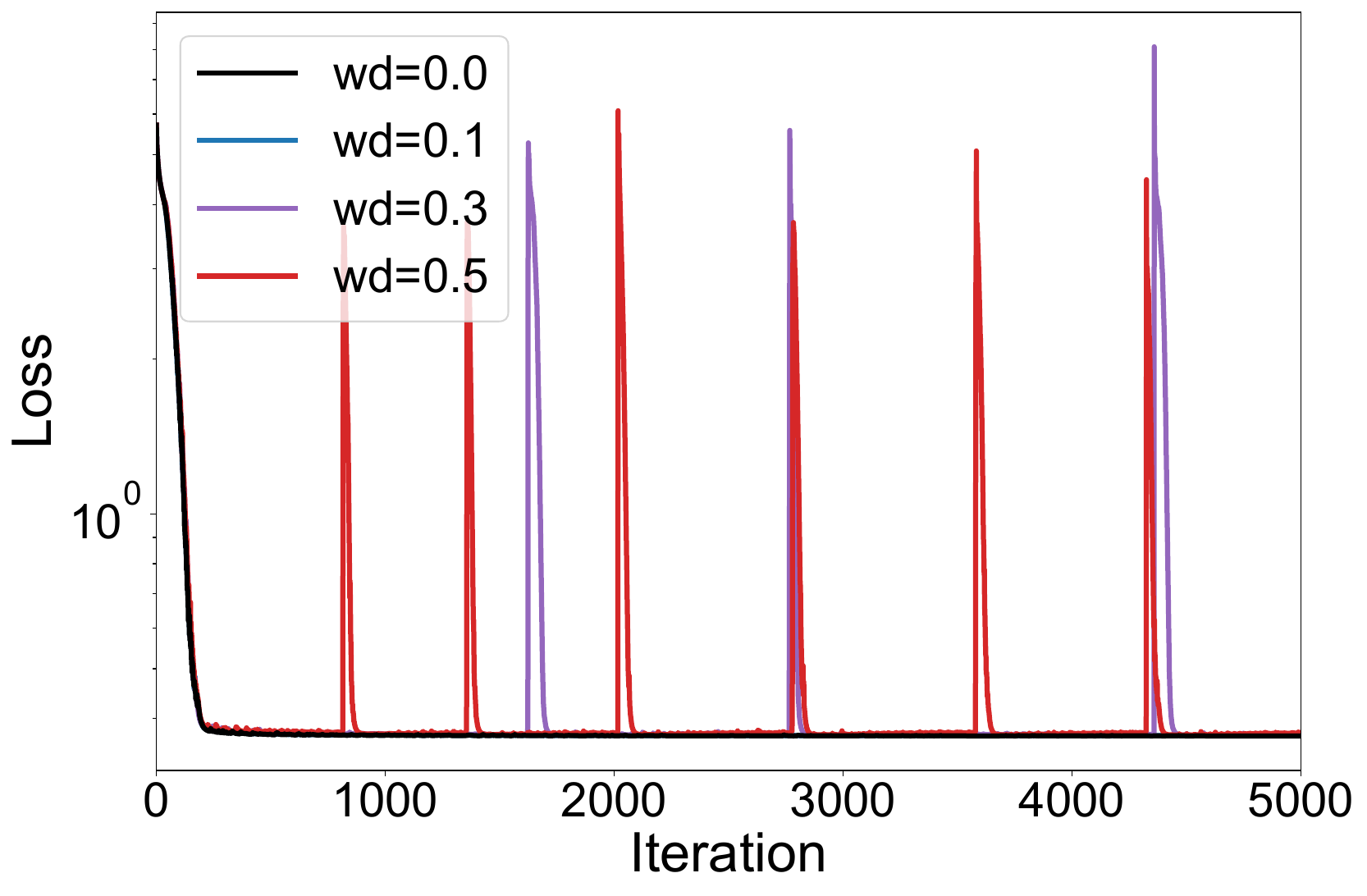}
        \caption{Pre-LayerNorm}
        \label{fig:transformer_pre_layernorm}
    \end{subfigure}
    \hfill
    \begin{subfigure}[t]{0.32\textwidth}
        \centering
        \includegraphics[width=\linewidth]{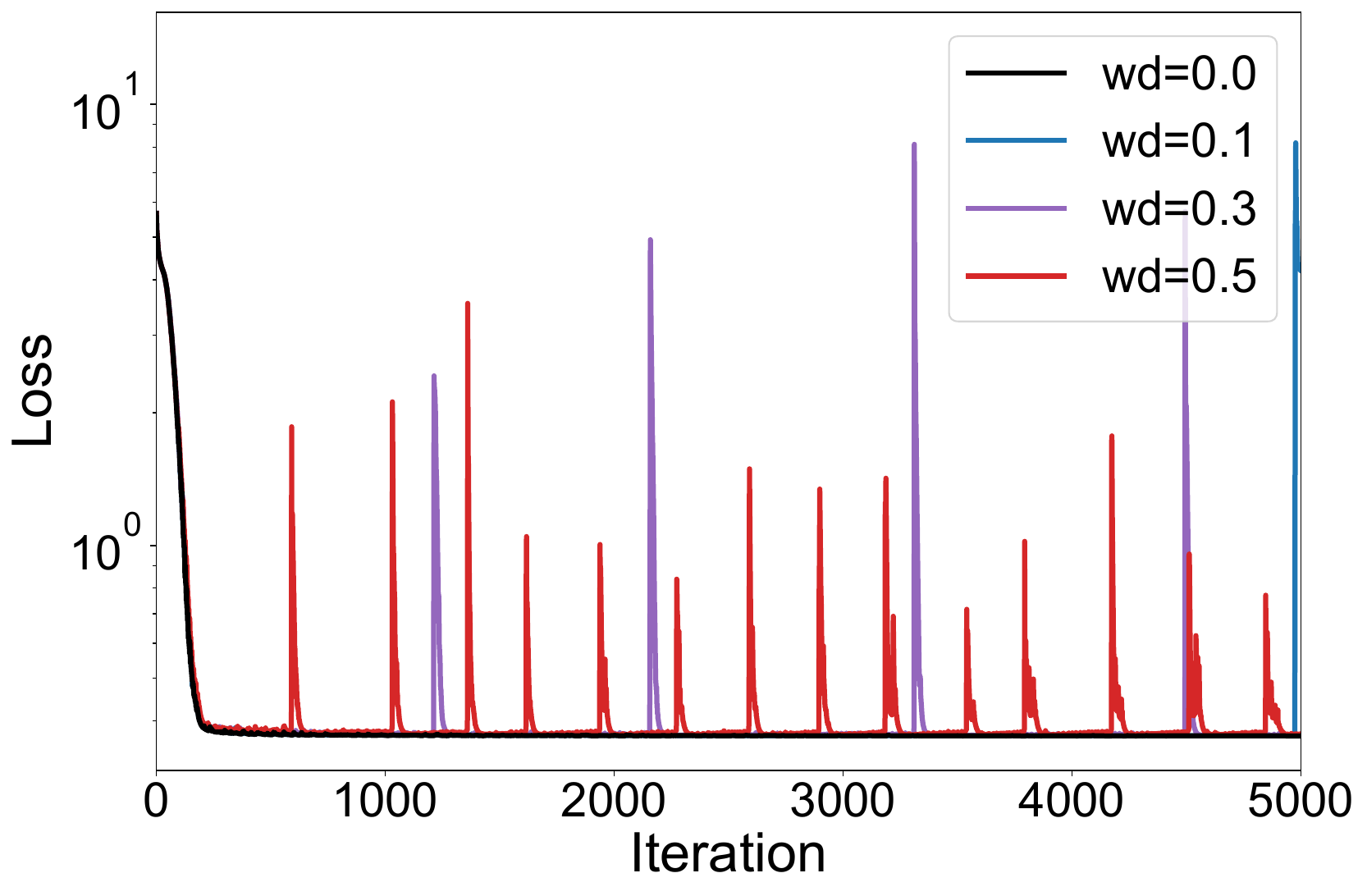}
        \caption{Post-LayerNorm}
        \label{fig:transformer_post_layernorm}
    \end{subfigure}
    \hfill
    \begin{subfigure}[t]{0.32\textwidth}
        \centering
        \includegraphics[width=\linewidth]{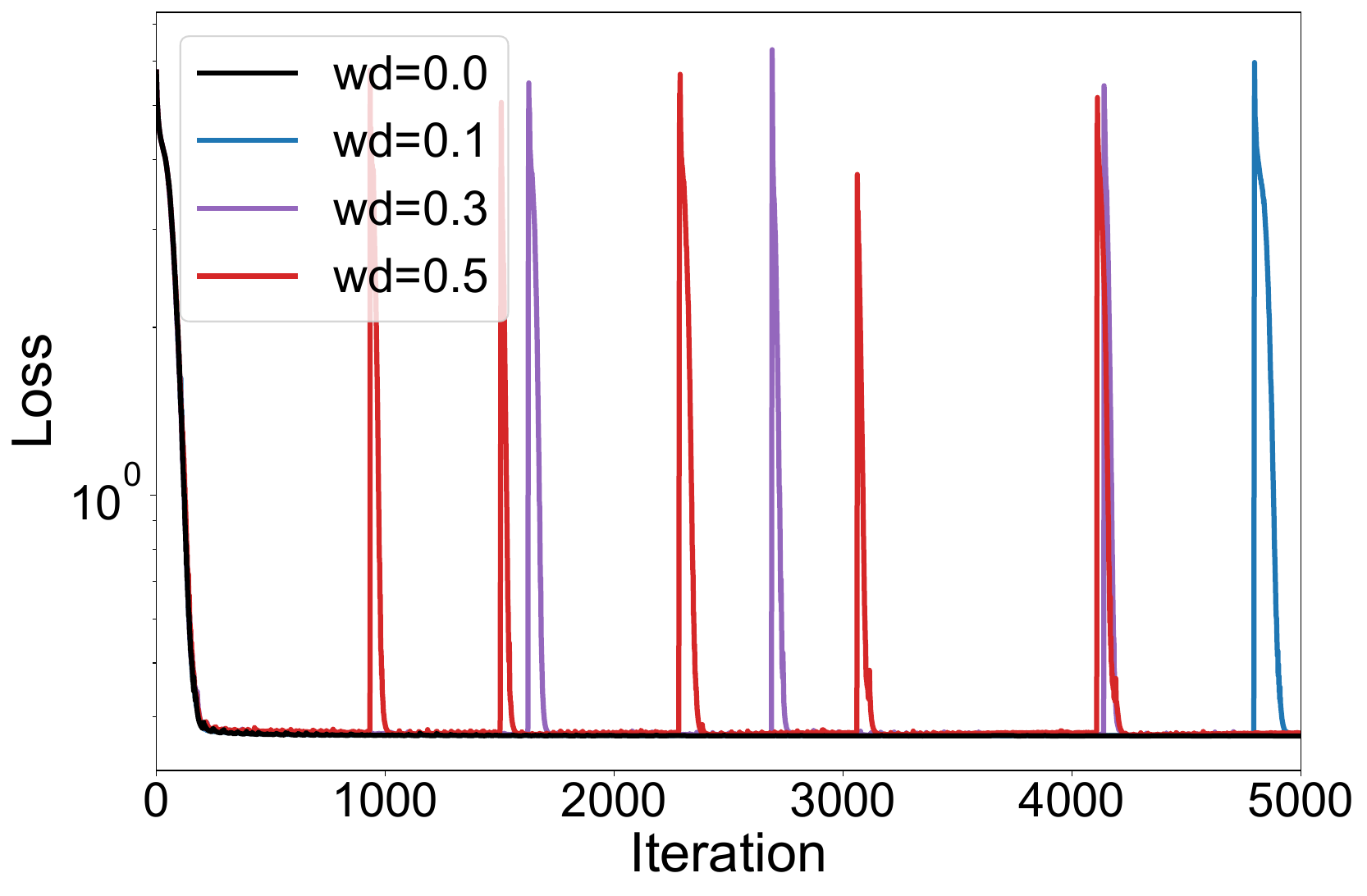}
        \caption{Sandwich-Norm}
        \label{fig:transformer_sandwich_norm}
    \end{subfigure}
    \begin{subfigure}[t]{0.32\textwidth}
        \centering
        \includegraphics[width=\linewidth]{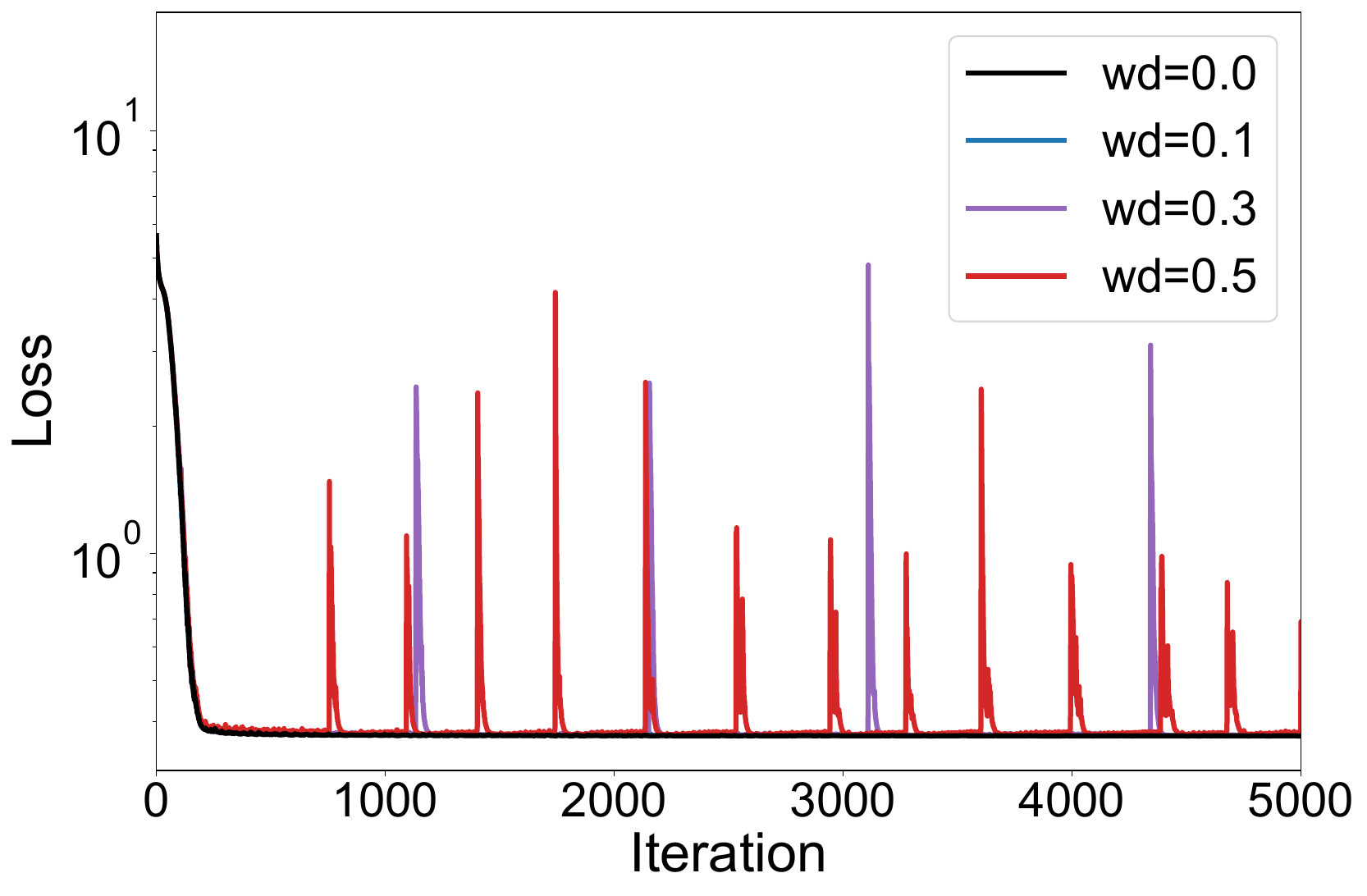}
        \caption{RMSNorm}
        \label{fig:transformer_rmsnorm}
    \end{subfigure}
    \hfill
    \begin{subfigure}[t]{0.32\textwidth}
        \centering
        \includegraphics[width=\linewidth]{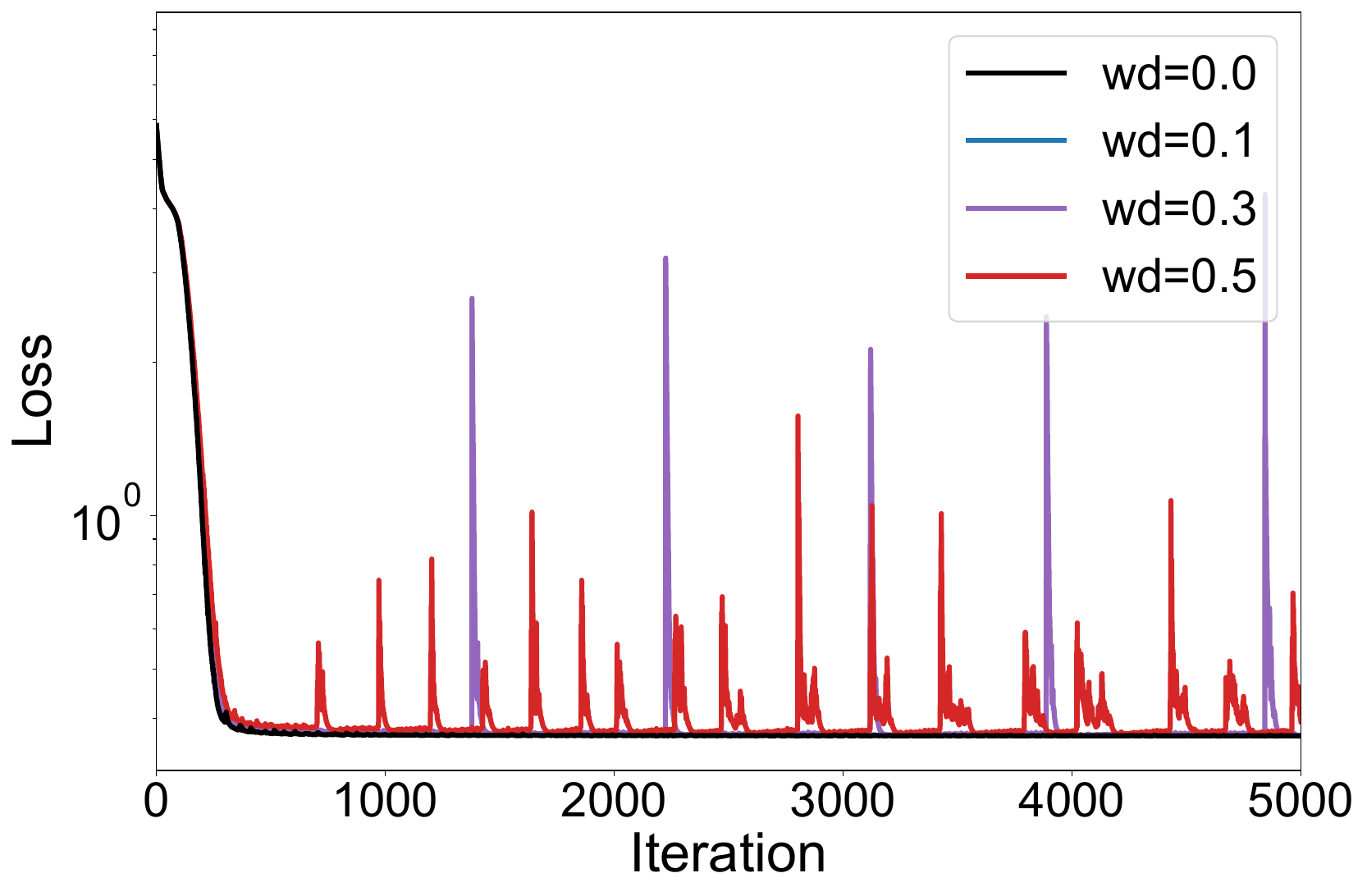}
        \caption{DeepNorm}
        \label{fig:transformer_deepnorm}
    \end{subfigure}
    \hfill
    \caption{
    Representative Transformer training trajectories under different
    normalization methods on the $3x\!\rightarrow\!x$ tasks.
    Panels (a)--(e) correspond to Pre-LayerNorm, Post-LayerNorm,
    Sandwich-Norm, RMSNorm, and DeepNorm, respectively.
    For each normalization method, stronger weight decay consistently
    leads to more frequent loss spikes across different random seeds.
    }
    \label{fig:transformer_norm_robustness}
\end{figure*}

Specifically, the same qualitative relationship between weight decay and loss spikes can be observed for Pre-LayerNorm
(Fig.~\ref{fig:transformer_norm_robustness}\subref{fig:transformer_pre_layernorm}),
Post-LayerNorm
(Fig.~\ref{fig:transformer_norm_robustness}\subref{fig:transformer_post_layernorm}),
Sandwich-Norm
(Fig.~\ref{fig:transformer_norm_robustness}\subref{fig:transformer_sandwich_norm}),
RMSNorm
(Fig.~\ref{fig:transformer_norm_robustness}\subref{fig:transformer_rmsnorm}),
and DeepNorm
(Fig.~\ref{fig:transformer_norm_robustness}\subref{fig:transformer_deepnorm}).

Although the placement of normalization layers and the presence of residual connections break strict scale invariance from a mathematical perspective, normalization may still reduce the sensitivity of the model output to the scales of certain parameters at the level of training dynamics. Weight decay can therefore continuously reduce these parameter norms while producing a comparatively limited immediate effect on the forward pass, potentially increasing the effective optimization step size and local curvature.

These empirical observations are qualitatively consistent with the mechanism described by our theoretical analysis. However, because strict scale invariance does not hold in these Transformer architectures, we do not claim that the current theoretical boundary directly characterizes their training dynamics. We therefore restrict the present discussion to empirical observations and leave a rigorous theoretical characterization of this phenomenon in Transformers to future work.

\end{document}